\documentclass{article}

\usepackage[preprint]{neurips_2021}

\usepackage[utf8]{inputenc} 
\usepackage[T1]{fontenc}    
\usepackage{hyperref}       
\usepackage{url}            
\usepackage{booktabs}       
\usepackage{amsfonts}       
\usepackage{nicefrac}       
\usepackage{microtype}      
\usepackage{xcolor}         
\usepackage[most]{tcolorbox}
\usepackage{enumitem}

\usepackage{amsmath}
\usepackage{amssymb}
\usepackage{mathtools}
\usepackage{amsthm}
\usepackage{graphicx}
\usepackage{subfigure}

\theoremstyle{plain}
\newtheorem{theorem}{Theorem}[section]

\newtheorem{lemma}[theorem]{Lemma}

\theoremstyle{definition}

\theoremstyle{remark}

\usepackage[capitalize,noabbrev]{cleveref}
\usepackage{framed}
\usepackage{multirow}
\usepackage{svg}
\usepackage{listings}
\usepackage{wrapfig}

\newcommand\our{\textsc{DeepNet}}
\newcommand\deepnorm{\textsc{DeepNorm}}
\newcommand\postln{Post-LN}
\newcommand\preln{Pre-LN}

\title{DeepNet: Scaling Transformers to 1,000 Layers}

\author{{Hongyu Wang\thanks{Equal contribution. Work was done during Hongyu's internship at Microsoft Research.}~~~~Shuming Ma\footnotemark[1]~~~~Li Dong~~~Shaohan Huang~~~Dongdong Zhang~~~Furu Wei\thanks{Corresponding author \textless{}\href{mailto:fuwei@microsoft.com}{fuwei@microsoft.com}\textgreater{} .}} \\
Microsoft Research \\
{\url{https://github.com/microsoft/unilm}}
}

\begin{document}

\maketitle

\begin{abstract}
In this paper, we propose a simple yet effective method to stabilize extremely deep Transformers. Specifically, we introduce a new normalization function (\deepnorm{}) to modify the residual connection in Transformer, accompanying with theoretically derived initialization.
In-depth theoretical analysis shows that model updates can be bounded in a stable way.
The proposed method combines the best of two worlds, i.e., good performance of \postln{} and stable training of \preln{}, making \deepnorm{} a preferred alternative.
We successfully scale Transformers up to 1,000 layers (i.e., 2,500 attention and feed-forward network sublayers) without difficulty, which is one order of magnitude deeper than previous deep Transformers. Remarkably, on a multilingual benchmark with 7,482 translation directions, our 200-layer model with 3.2B parameters significantly outperforms the 48-layer state-of-the-art model with 12B parameters by 5 BLEU points, which indicates a promising scaling direction.
\end{abstract}


\begin{figure}[htbp]
\centering
\includegraphics[width=\columnwidth]{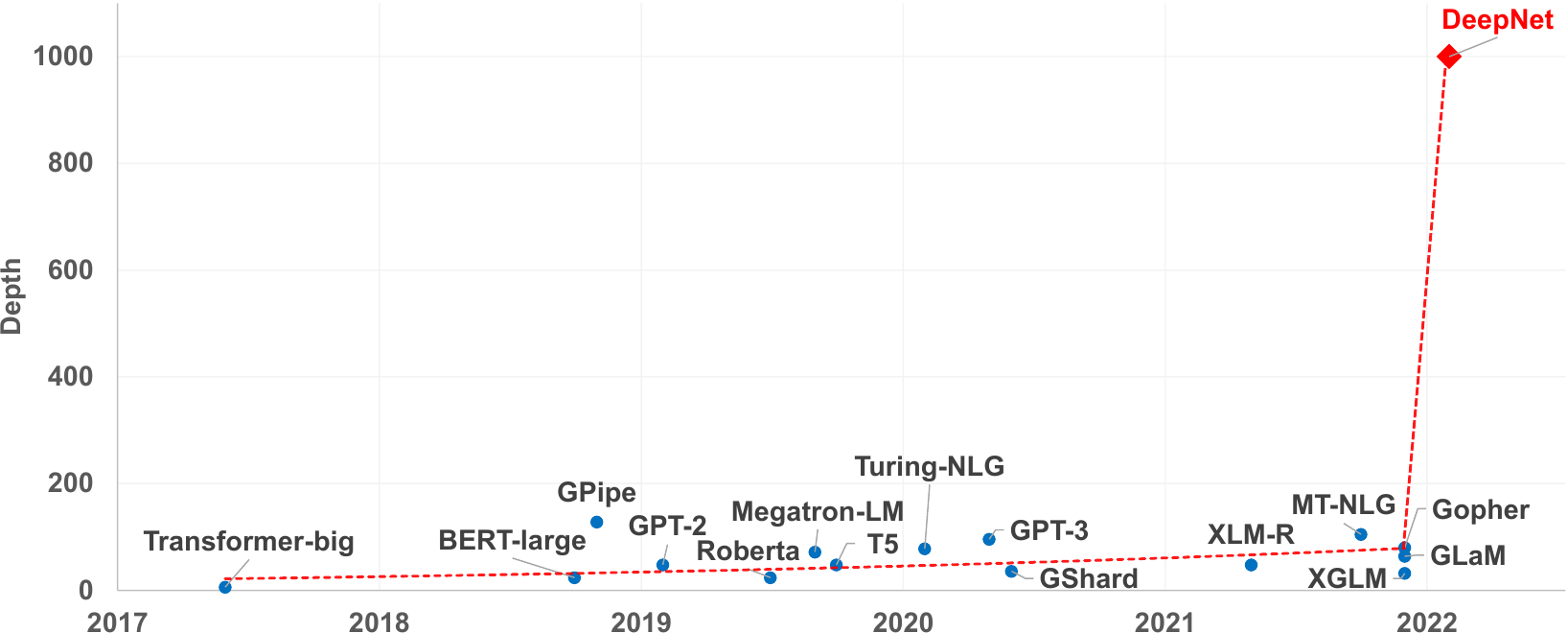}
\caption{Trend of Transformer depths of state-of-the-art NLP models over time.}
\label{depth_trend}
\end{figure}

\newpage
\section{Introduction}
\label{sec:intro}

Recent years have witnessed a trend towards large-scale Transformer~\citep{transformer} models. The capacity has substantially increased from millions of parameters~\citep{JacobDevlin2018BERTPO,xlmr} to billions~\citep{gpt-2,gpt3,gpipe,t5,gshard,gopher,xglm,mt-nlg}, and even trillions~\citep{glam}. Large-scale models yield state-of-the-art performance on a wide range of tasks, and show impressive abilities in few-shot and zero-shot learning. Despite an enormous number of parameters, their depths (as shown in \cref{depth_trend}) are limited by the training instability of Transformers.

\citet{ToanQNguyen2019TransformersWT} find that pre-norm residual connections (\preln{}) improve the stability of Transformers based on post-norm connections (\postln{}). However, the gradients of \preln{} at bottom layers tend to be larger than at top layers~\citep{Normformer2021}, leading to a degradation in performance compared with \postln{}.
In order to alleviate the above issue, there have been efforts on improving the optimization of deep Transformer by means of better initialization~\citep{BiaoZhang2019ImprovingDT,HongyiZhang2019FixupIR,XiaoShiHuang2020ImprovingTO}, or better architecture~\citep{Wang2019DLCL,LiyuanLiu2020UnderstandingTD,rezero2020,Normformer2021}.
These approaches can stabilize a Transformer model with up to hundreds of layers. Yet, none of previous methods has been successfully scaled to 1,000 layers.

Our aim is to improve the training stability of Transformers and scale the model depth by orders of magnitude. To this end, we study the cause of unstable optimization, finding the exploding model update is responsible for the instability.
Motivated by the above observation, we introduce a new normalization function (\deepnorm{}) at residual connections~\citep{resnet}, which has theoretical justification of bounding the model update by a constant.
The proposed method is simple yet effective, with just lines of code change.
The approach improves the stability of Transformers so that we are able to scale model depth to more than 1,000 layers.
Moreover, experimental results show that \deepnorm{} combines the best of two worlds, i.e., good performance of \postln{} and stable training of \preln{}.
The proposed method can be a preferred alternative of Transformers, not only for extremely deep (such as >1000 layers) models, but also for existing large models.
Notably, our 200-layer model with 3.2B parameters achieves 5 BLEU improvement on a massively multilingual machine translation benchmark compared to state-of-the-art model~\citep{m2m100} with 48 layers and 12B model size.

\section{TL;DR for Practitioners}

\begin{figure}[htbp]
\centering
\begin{minipage}[t]{0.44\columnwidth}
\vspace{0pt}
\begin{lstlisting}[language=python,mathescape]
  def deepnorm(x):
      return LayerNorm(x * <@\textcolor{red}{$\alpha$}@> + f(x))

  def deepnorm_init(w):
      if w is ['ffn', 'v_proj', 'out_proj']:
          nn.init.xavier_normal_(w, gain=<@\textcolor{red}{$\beta$}@>)
      elif w is ['q_proj', 'k_proj']:
          nn.init.xavier_normal_(w, gain=1)
\end{lstlisting}
\end{minipage}
\hfill
\begin{minipage}[t]{0.545\columnwidth}
\vspace{0pt}
\resizebox{\columnwidth}{!}{
\begin{tabular}{l|cc|cc}
\toprule
\multirow{2}{*}{\textbf{Architectures}} & \multicolumn{2}{c|}{\textbf{Encoder}} & \multicolumn{2}{c}{\textbf{Decoder}} \\
&  $\alpha$ & $\beta$ & $\alpha$ & $\beta$ \\
\midrule
\text{Encoder-only} & \multirow{2}{*}{$(2N)^{\frac{1}{4}}$} & \multirow{2}{*}{$(8N)^{-\frac{1}{4}}$} & \multirow{2}{*}{-} & \multirow{2}{*}{-} \\
\text{(e.g., BERT)} & & & & \\
\text{Decoder-only} &  \multirow{2}{*}{-} &  \multirow{2}{*}{-} & \multirow{2}{*}{$(2M)^{\frac{1}{4}}$} &  \multirow{2}{*}{$(8M)^{-\frac{1}{4}}$} \\
\text{(e.g., GPT)} & & & & \\
\text{Encoder-decoder} &  \multirow{2}{*}{$0.81(N^4M)^{\frac{1}{16}}$} &  \multirow{2}{*}{$0.87(N^4M)^{-\frac{1}{16}}$} &  \multirow{2}{*}{$(3M)^{\frac{1}{4}}$} &  \multirow{2}{*}{$(12M)^{-\frac{1}{4}}$} \\
\text{(e.g., NMT, T5)} & & & & \\
\bottomrule
\end{tabular}
}
\end{minipage}
\caption{(a) Pseudocode for \deepnorm{}. We take Xavier initialization~\citep{xavier} as an example, and it can be replaced with other standard initialization. Notice that $\alpha$ is a constant. (b) Parameters of \deepnorm{} for different architectures ($N$-layer encoder, $M$-layer decoder).}
\label{implementation}
\end{figure}

As shown in \cref{implementation}, it is simple to implement our method based on Transformers with \postln{}.
Compared to \postln{}, \deepnorm{} up-scales the residual connection before performing layer normalization. Besides, we down-scale the parameters during initialization.
Notably, we only scale the weights of feed-forward networks, as well as the value projection and the output projection of attention layers.
Moreover, the scales of residual connection and initialization are dependent on the architecture (\cref{implementation}). We provide more details in~\cref{set:imple}.

\section{Instability of Deep Transformer}
\label{instability}

We study the causes of the instability for deep Transformers. Our analysis begins with the observation: better initialization methods stabilize the training of Transformer. This has also been verified by previous work~\citep{BiaoZhang2019ImprovingDT,XiaoShiHuang2020ImprovingTO, PengXu2021OptimizingDT}. Therefore, we study the training process of \postln{} with or without proper initialization.
With better initialization, we down-scale the weights of $l$-th layer by $k_l = N - l + 1, l \in [1, N]$ after performing Xavier initialization. For example, the output projection $W^{l}_{o}$ of FFN in $l$-th layer is initialized as:

\begin{equation}
    W^{l}_{o} \backsim \mathcal{N} \left(0, \frac{1}{k_l^2d'} \right) \notag,
\end{equation}

where $d'$ is an average of input and output dimensions. We name this model \postln{}-init. Notice that different from the prior work~\citep{BiaoZhang2019ImprovingDT}, we narrow the scale of lower layers instead of the higher layers. We believe that it helps to separate the effect of the gradient scale from the model update. Besides, \postln{}-init has the same architecture as \postln{}, which eliminates the impact from the architecture.

\begin{figure}[t]
\begin{center}
{
\includegraphics[width=0.3\columnwidth]{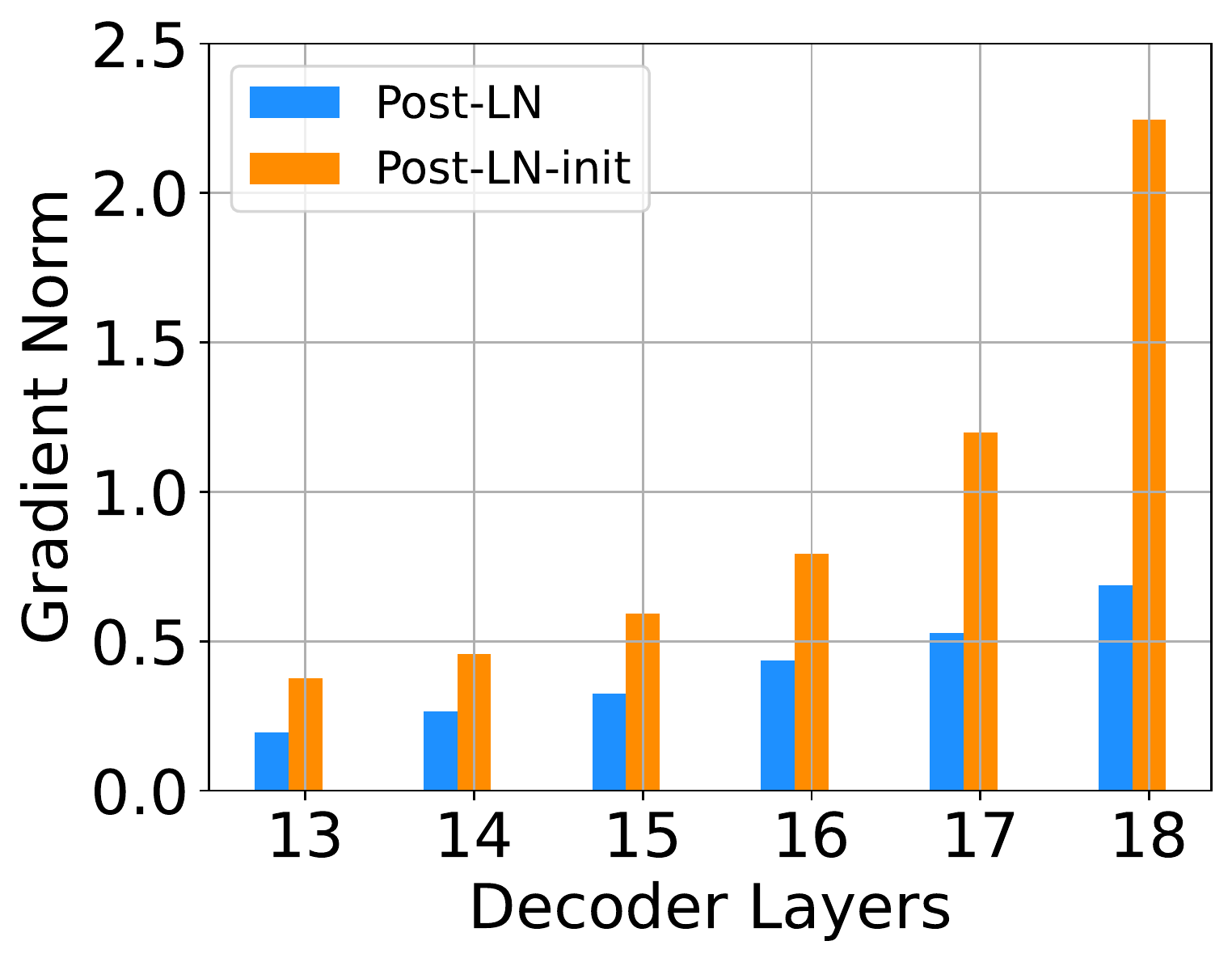}
}
{
\includegraphics[width=0.3\columnwidth]{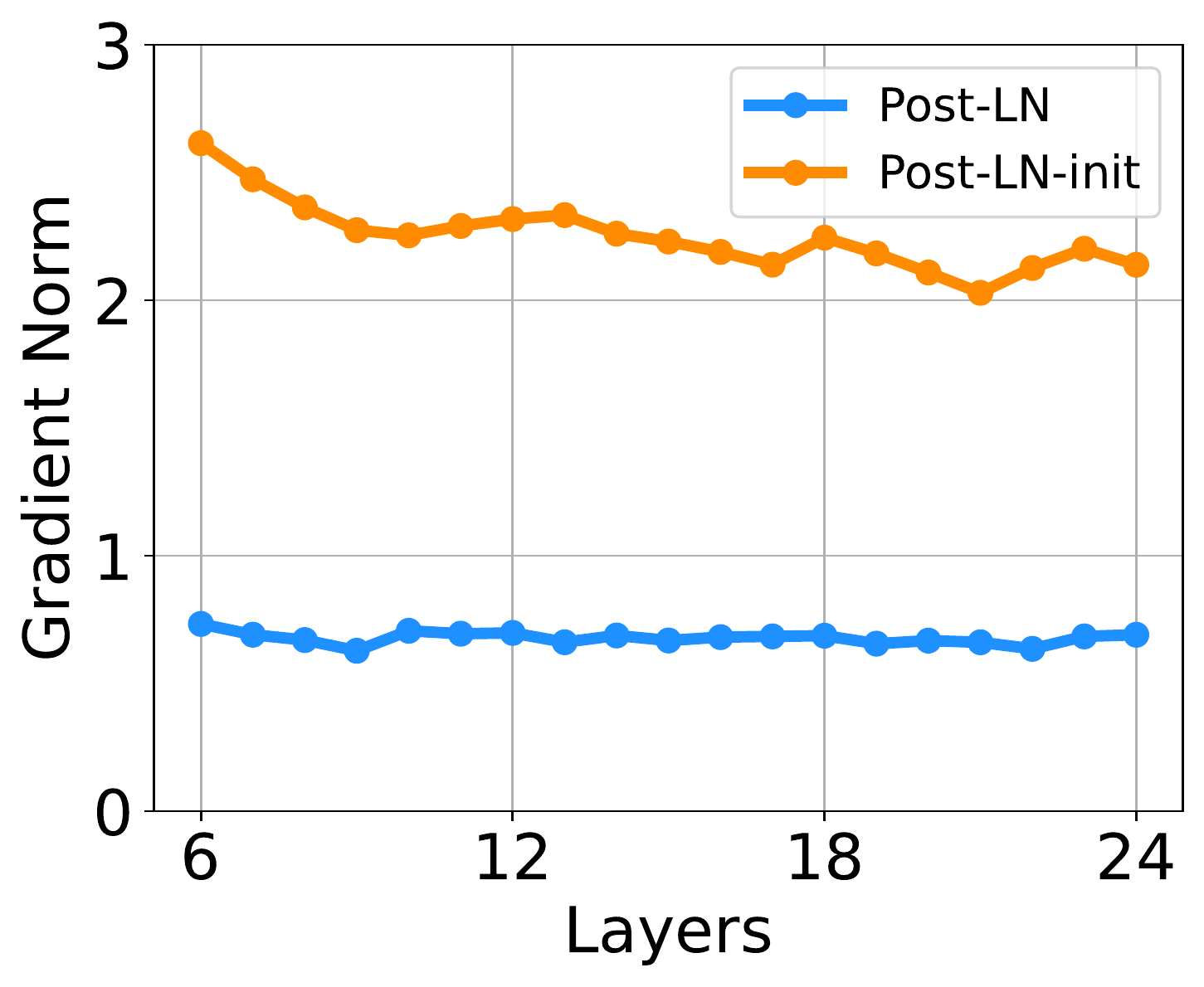}
}
{
\includegraphics[width=0.3\columnwidth]{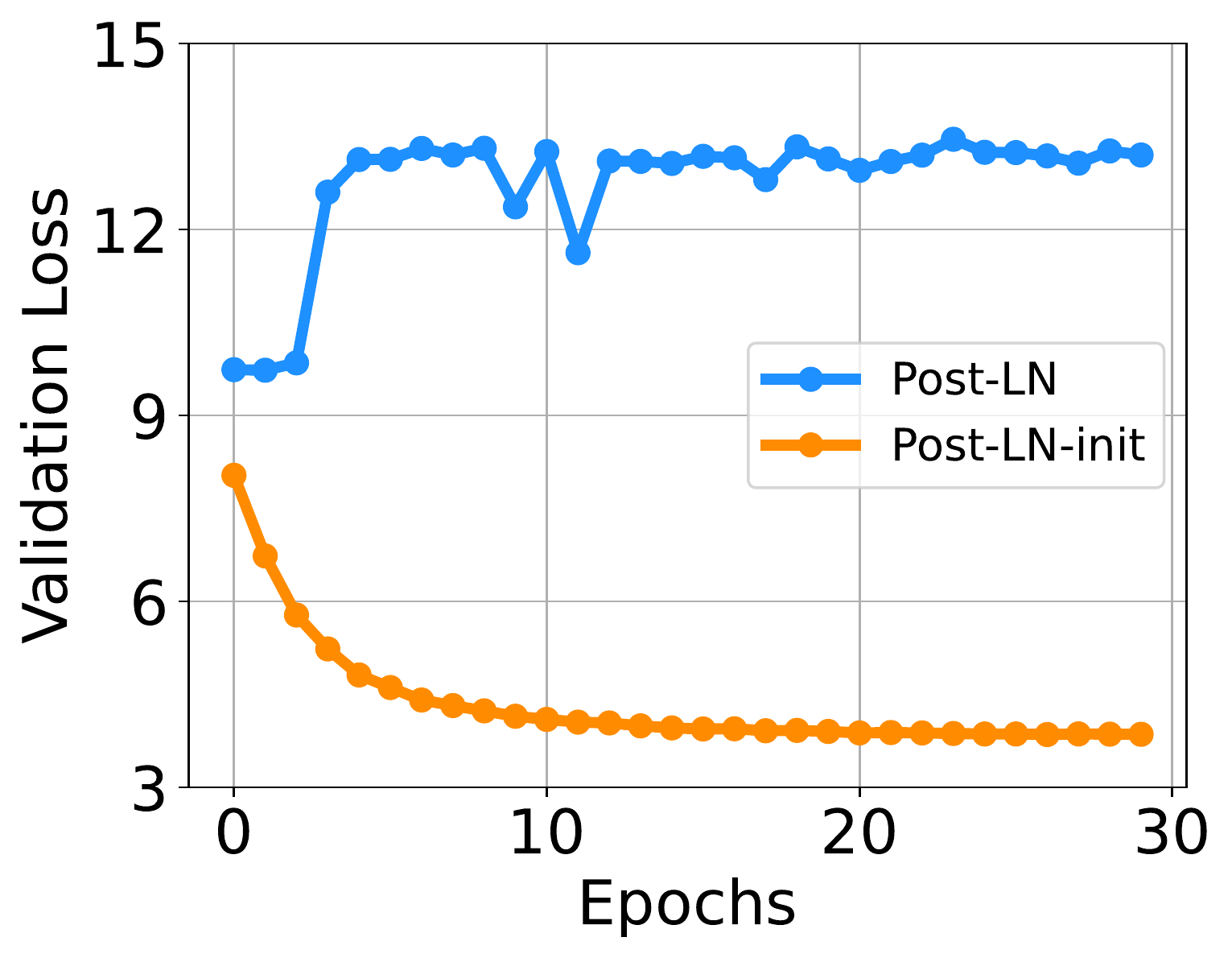}
}
\caption{(a) Gradient norm in the top layers of 18L-18L models. (b) Gradient norm in the last layer of the models with depths varying from 6L-6L to 24L-24L. (c) Validation loss curves of 18L-18L models.}\label{grad_and_valid_loss}
\end{center}
\end{figure}

\begin{figure}[t]
\begin{center}
\subfigure[Accumulated model update]{
\includegraphics[width=0.3\columnwidth]{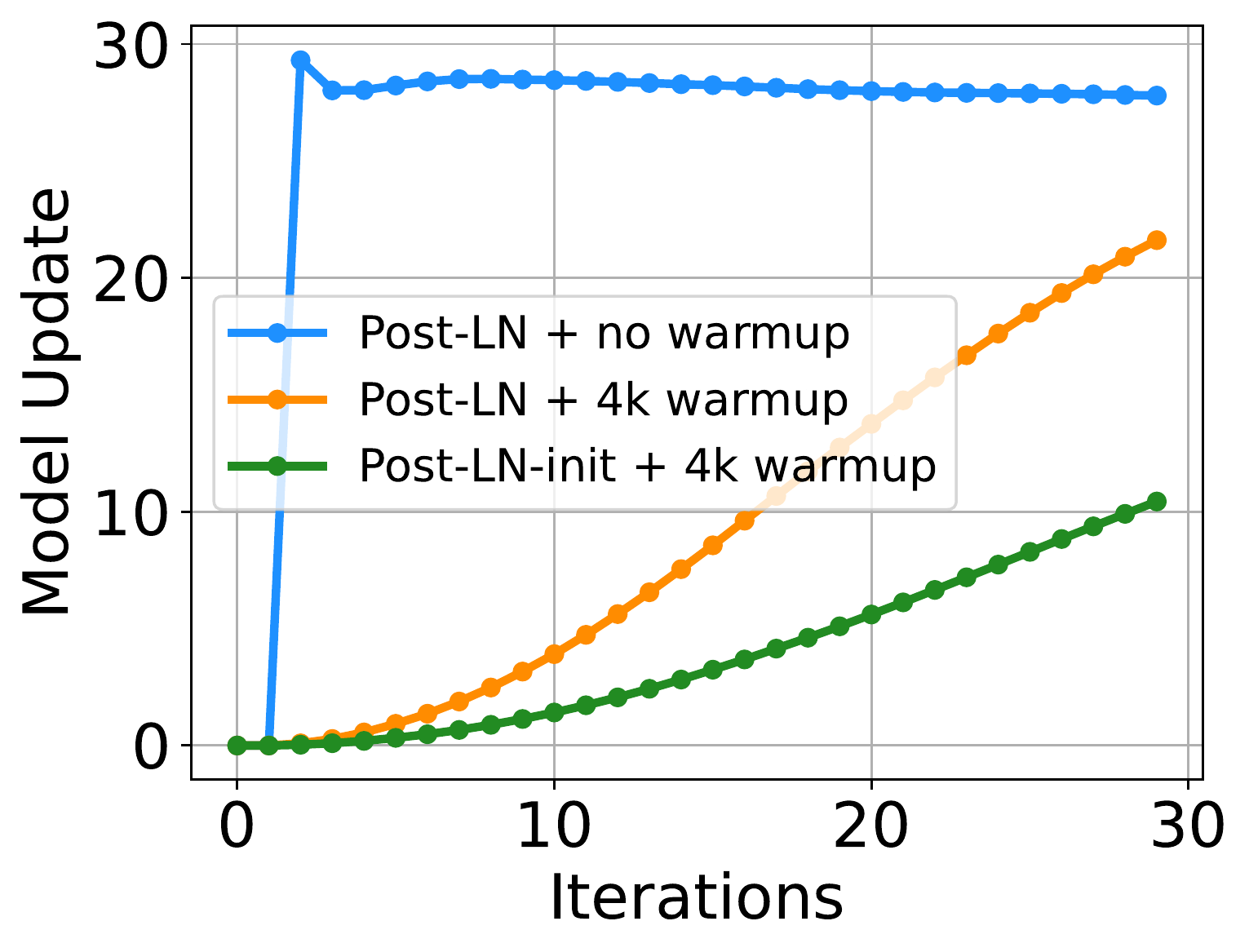}
\label{plot_model_update}
}
\subfigure[Input from FFN to LN]{
\includegraphics[width=0.3\columnwidth]{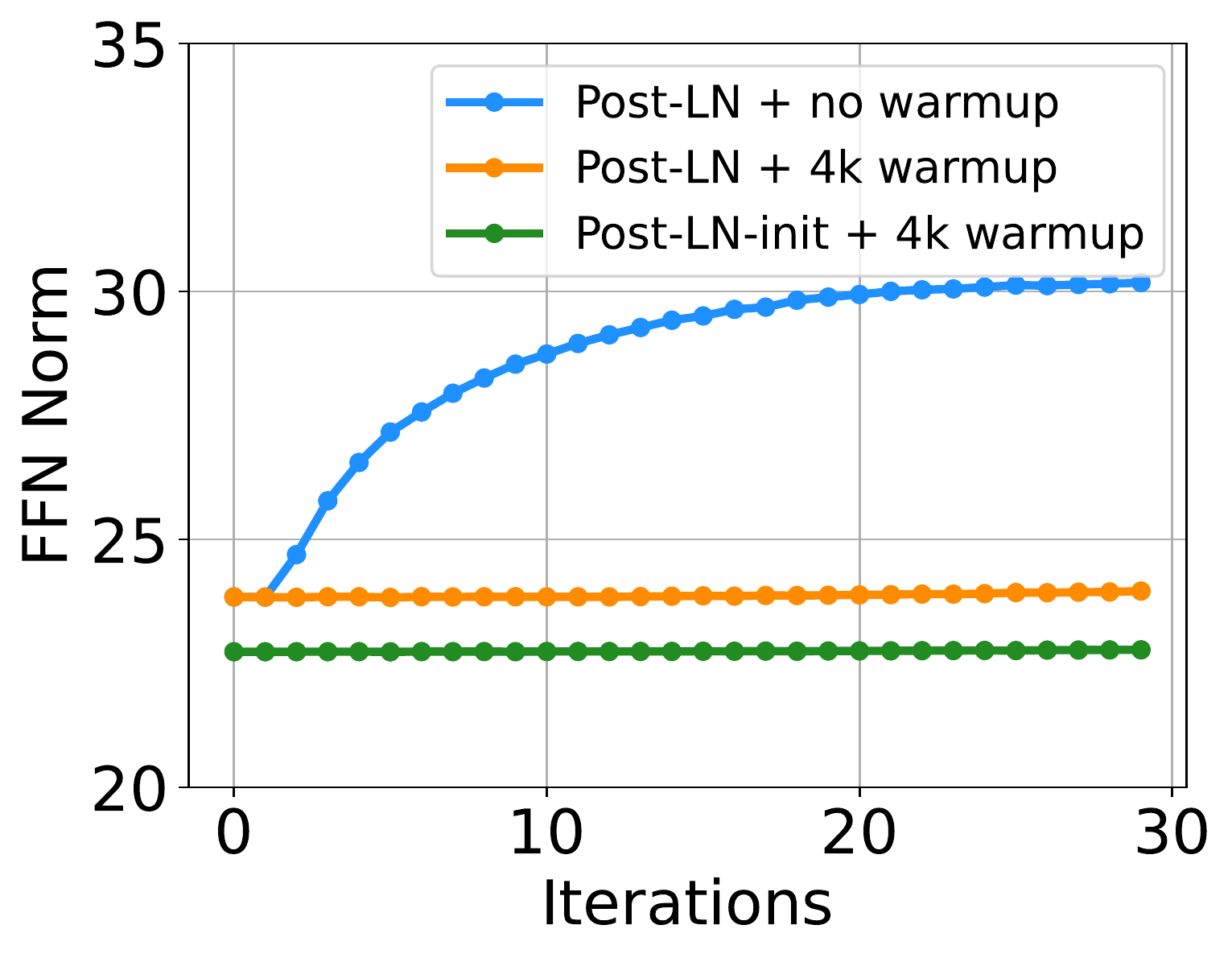}
\label{plot_mean_norm_ffn}
}
\subfigure[Input from attention to LN]{
\includegraphics[width=0.3\columnwidth]{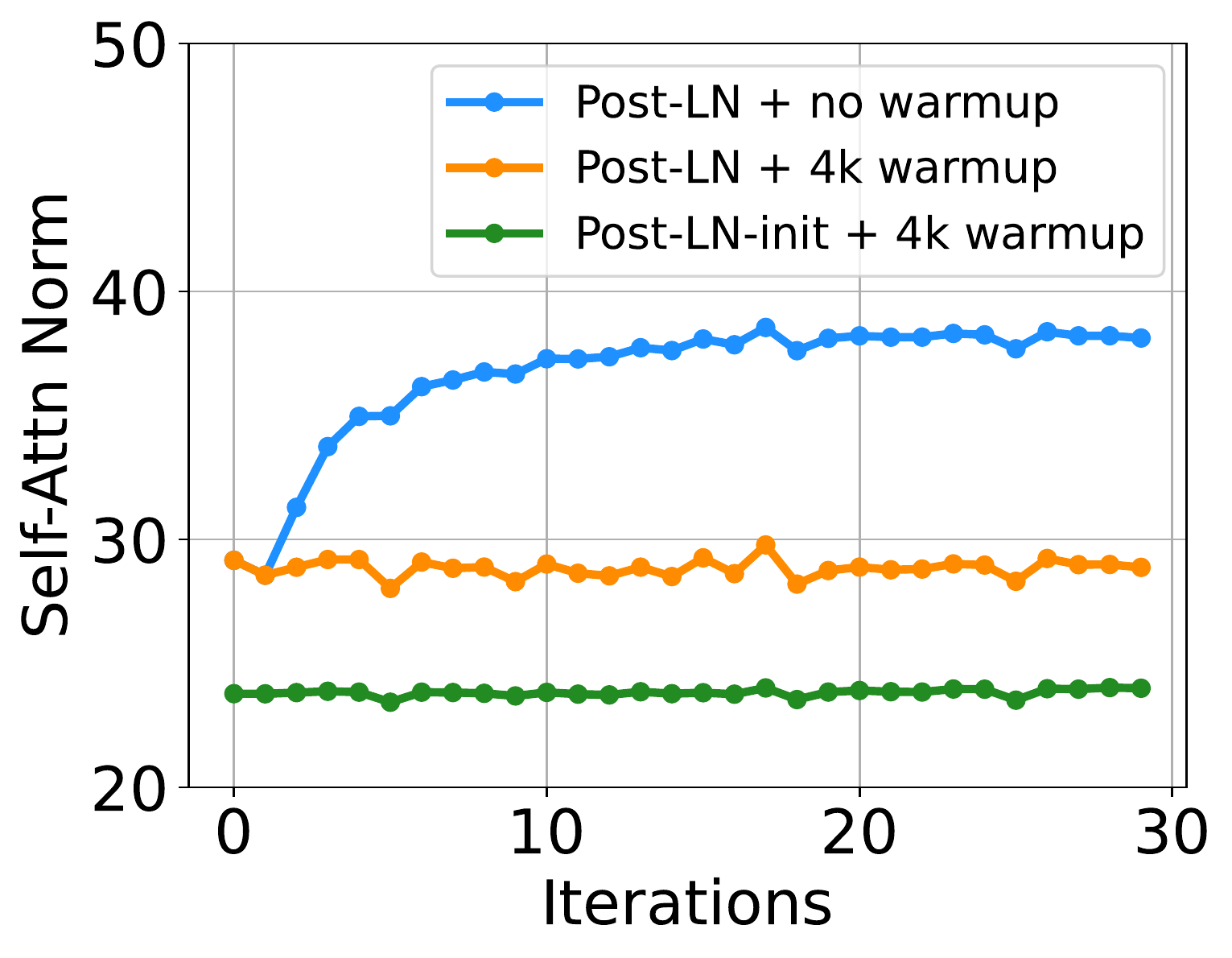}
\label{plot_mean_norm_self_attn}
}
\subfigure[Gradient norm in all decoder layers]{
\includegraphics[width=0.9\columnwidth]{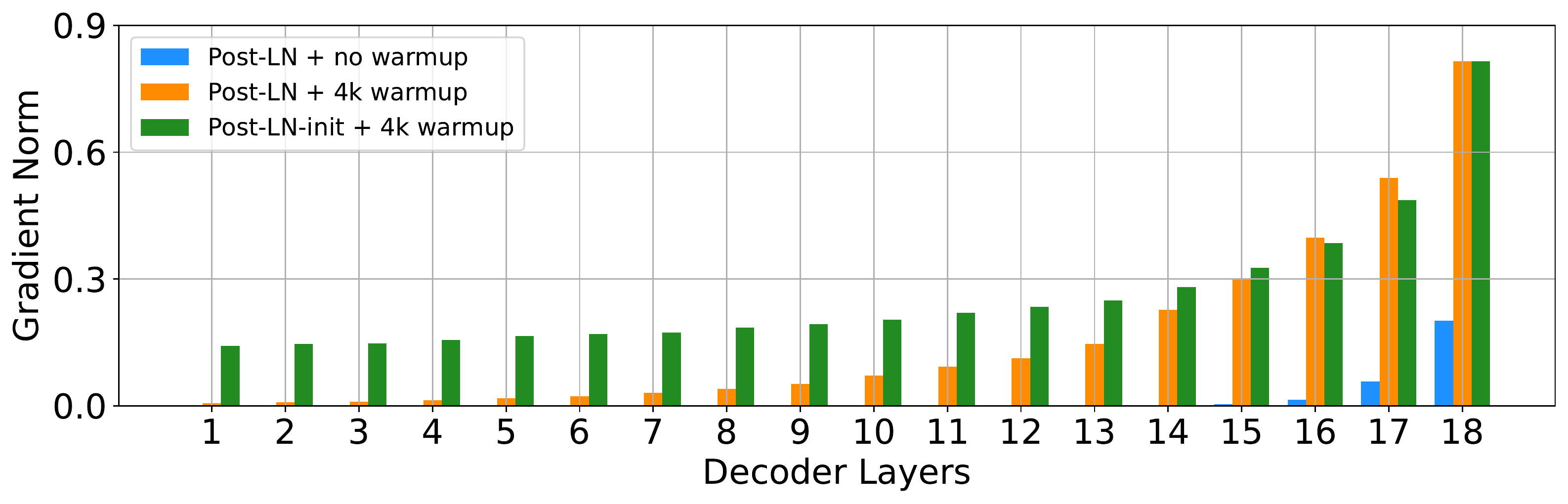}
\label{his_grad_vanishing}
}

\caption{Visualization of the model update, the average input of LNs, and the gradients for the 18L-18L models at the early stage of training.}
\end{center}
\end{figure}

We train 18L-18L \postln{} and 18L-18L \postln{}-init on the IWSLT-14 De-En machine translation dataset.
\cref{grad_and_valid_loss} visualizes their gradients and validation loss curves.
As shown in \cref{grad_and_valid_loss}(c), \postln{}-init converged while \postln{} did not.
\postln{}-init has an even larger gradient norm in the last several layers, although its weights have been scaled down. Furthermore, we visualize the gradient norm of the last decoder layer with varying model depth from 6L-6L to 24L-24L.
\cref{grad_and_valid_loss} shows that the gradient norm of \postln{}-init in the last layer is still much larger than that of \postln{}, regardless of model depth.
It concludes that the exploding gradients in deep layers should not be the root cause of instability of \postln{}, while the scale of model update tends to account for it.

Then we demonstrate that the instability of \postln{} comes from a chain of several issues, including gradient vanishing as well as too large model updates. As shown in~\cref{plot_model_update}, we first visualize the norm of model update $|| \Delta F ||$ at the early stage of training:
\begin{equation*}
|| \Delta F || = || F(x, \theta_i) - F(x, \theta_0) ||, 
\end{equation*} 
where $x$ and $\theta_i$ denotes input, and model parameters after $i$-th updates.
\postln{} has an exploding update at the very beginning of training, and then nearly no update shortly. It indicates that the model has been stuck in a spurious local optima. Both warm-up and better initialization help alleviate this issue, enabling the model to update smoothly.
When the update explodes, the inputs to LN become large (see \cref{plot_mean_norm_ffn} and \cref{plot_mean_norm_self_attn}).
According to the theoretical analysis from \citet{RuibinXiong2020OnLN}, the magnitude of gradient through LN is inversely proportional to the magnitude of its input:
\begin{equation*}
\label{xiong_eq}
||\cfrac{\partial LN(x)}{\partial x}|| = \mathcal{O}(\cfrac{\sqrt{d}}{||x||}).
\end{equation*} 
\cref{plot_mean_norm_ffn} and \cref{plot_mean_norm_self_attn} show that $||x||$ is significantly larger than $\sqrt{d}$ $(d=512)$ without warm-up or proper initialization. This explains the gradient vanishing problem occurred in the training of \postln{} (see \cref{his_grad_vanishing}).

Above all, the instability starts from the large model update at the beginning of training. It renders the model trapped in a bad local optima, which in turn increases the magnitude of inputs to each LN. 
As training continues, the gradient through LN becomes increasingly small, thus resulting in severe gradient vanishing.
The vanishing gradients make it difficult to escape from the local optima, and further destabilize the optimization.
On the contrary, \postln{}-init has relatively small updates, and the inputs to LN are stable. This relieves suffering from gradient vanishing, making optimization more stable.

\section{\our{}: Extremely Deep Transformers}
\label{sec:method}

In this section, we introduce our extremely deep Transformers named \our{}.
It can stabilize the optimization by mitigating the exploding model update problem.
We first provide the estimation of the expected magnitude of \our{}'s model update.
Then we provide the theoretical analysis to show that its updates can be bounded by a constant with our proposed \deepnorm{}.

\subsection{Architecture}

\our{} is based on the Transformer architecture. Compared to the vanilla Transformer, it uses our new \deepnorm{}, instead of \postln{}, for each sub-layer. The formulation of \deepnorm{} can be written as:
\begin{equation*}
    x_{l+1} = LN(\alpha x_l + G_l(x_l, \theta_l)),
\end{equation*}
where $\alpha$ is a constant, and $G_l(x_l, \theta_l)$ is the function of the $l$-th Transformer sub-layer (i.e., attention or feed-forward network) with parameters $\theta_l$.
Besides, \our{} scales the weights $\theta_l$ inside residual branches by $\beta$.
Notably, both $\alpha$ and $\beta$ are constants that only depend on the architecture, and we provide the derivation in~\cref{set:imple}.

\subsection{Expected Magnitude of Model Update}

Attention is an important part of Transformer. Without loss of generality, we study the 1-head case. Let $Q, K, V \in \mathbf{R}^{n \times d}$ denote the query, key, value, respectively. $W^Q, W^K, W^V \in \mathbf{R}^{d \times d_{k}}$ are the input projection matrices, and $W^O \in \mathbf{R}^{d_{k} \times d}$ is the output projection matrix. Then, the attention module can be formulated as:
\begin{equation*}
    Attn(Q, K, V) = softmax(\cfrac{Q W^Q (K W^K)^T}{\sqrt{d_k}}) V W^V W^O
\end{equation*}

We study the magnitude of the attention module. Lemma~\ref{lem:softmax} proves that $W^Q$ and $W^K$ do not change the bound of attention output's magnitude.

\begin{lemma}
\label{lem:softmax}
     Given $\mathbf{X} = (\mathbf{x_1}, \mathbf{x_2}, ... \mathbf{x_n})^T \in \mathbf{R}^{n \times d}$, where $var(\mathbf{x_i}) = 1$, $mean(\mathbf{x_i}) = 0$ and $q_i \in \mathbf{R}$ for all $i \in [1, n]$, it satisfies that
     \begin{equation*}
         softmax(q_1, q_2, ..., q_n) \mathbf{X} \overset{\Theta}{=} \mathbf{x_i},
     \end{equation*}
     where $\overset{\Theta}{=}$ stands for equal bound of magnitude.
\end{lemma}

In other words, the magnitude of attention output only depends on the value and output projection: $Attn(Q, K, V) \overset{\Theta}{=} V W^V W^O$. In this work, we only consider the magnitude of model update, so it is sufficiently instructive to study the case where the hidden dimension equals to $1$. For simplicity, we reduce the matrices $W^V, W^O$ to the scalars $v, w$, which means $Attn(Q, K, V) \overset{\Theta}{=} vwV$. Similarly, we have $FFN(X) \overset{\Theta}{=} vwX$, where $v, w$ denotes the parameters of the feed-forward network.

We define the model update as $||\Delta F|| = || F(x, \theta^*) - F(x, \theta) ||$. Based on the analysis above, we have the following theorem to characterize $||\Delta F||$'s magnitude of an $N$-layer \our{} with $N$ attentions and FFNs.

\begin{theorem}
\label{thm:encoder}
Given an $N$-layer \our{} $F(x, \mathbf{\theta})$ ($\theta = \{\theta_1, \theta_2, ..., \theta_{2N} \} $), where $\theta_{2l-1}$ and ${\theta_{2l}}$ denote the parameters of self-attention and FFN in $l$-th layer, and each sub-layer is normalized with \deepnorm{}: $x_{l+1} = LN(\alpha x_l + G_l(x_l, \theta_l))$, $ ||\Delta F||$ satisfies: 
\begin{equation*}
    || \Delta F || \le \sum_{i=1}^{2N} \cfrac{\sqrt{v_{i}^2 + w_{i}^2}}{\alpha} || \theta_{i}^* - \theta_{i} ||
\end{equation*}
\end{theorem}

Vanilla \postln{} can be regarded as a special case of \our{}, where $\alpha = 1$ and $v_l = w_l = 1$ at Xavier initialization~\citep{xavier}. Based on \cref{thm:encoder}, we have $||\Delta F|| = \mathcal{O}(\sum_{i=1}^{2N} || \theta_{i}^* - \theta_{i} ||)$ for vanilla \postln{}. It shows that the model tends to accumulate the update of each sub-layer, which leads to exploding magnitude of model's update and destabilizes the optimization at the early stage. This explains our findings in~\cref{instability}. 

Besides, \cref{thm:encoder} also explains why warm-ups and smaller initialization can stabilize the training of \postln{}. Warm-ups can reduce the magnitude of the model update by decreasing $|| \theta_i^* - \theta_i ||$, while smaller initialization lowers $\sqrt{v_i^2 + w_i^2}$.

Furthermore, we study the magnitude of \our{} with an $N$-layer encoder and an $M$-layer decoder. Let $F_{ed}(x, y, \theta_e, \theta_d)$ denotes the model, where $x, y$ is the input of encoder and decoder. $\theta_e$ follows the same definition as $\theta$ in \cref{thm:encoder}. $\theta_d = \{\theta_{d1}, \theta_{d2}, ..., \theta_{d,3M} \} $ stands for the parameters of self-attentions, cross-attentions, and FFNs. We use \{$\alpha_e$, $G_{el}$\} and \{$\alpha_d$, $G_{dl}$\} to distinguish the notations between the encoder and the decoder. The following theorem shows the expected magnitude of the encoder-decoder's model update $ || \Delta F_{ed} || =  || F_{ed}(x, y, \theta_e^*, \theta_d^*) - F_{ed}(x, y, \theta_e, \theta_d)||$.

\begin{theorem}
\label{thm: en2de}
Given an encoder-decoder \our{} $F_{ed}(x, y, \theta_e, \theta_d)$ with N encoder layers and M decoder layers, where each encoder sub-layer is normalized as $x_{l+1} = LN(\alpha_e x_l + G_{el}(x_l, \theta_{el}))$, and the decoder sub-layer is normalized as $x_{l+1} = LN(\alpha_d x_l + G_{dl}(x_l, \theta_{dl}))$, $ ||\Delta F_{ed}||$ satisfies:
\begin{align}
\label{eq: en2de}
    || \Delta F_{ed} || 
    &\le \sum_{j=1}^{M} \cfrac{v_{d,3j-1}w_{d, 3j-1}}{\alpha_d}\sum_{i=1}^{2N} \cfrac{\sqrt{v_{ei}^2 +  w_{ei}^2}}{\alpha_e} || \theta_{ei}^* - \theta_{ei}|| \notag \\
    &\quad + \sum_{j=1}^{3M} \cfrac{\sqrt{v_{dj}^2 + w_{dj}^2}}{\alpha_{d}} ||\theta_{dj}^* - \theta_{dj}||
\end{align}
\end{theorem}

The vanilla encoder-decoder model satisfies that all of \{$\alpha_e$, $\alpha_d$, $v_{ei}$, $w_{ei}$, $v_{di}$, $w_{di}$\} equal to $1$, so we have $||\Delta F_{ed}|| = \mathcal{O}(M\sum_{i=1}^{2N} || \theta_{ei}^* - \theta_{ei} || + \sum_{j=1}^{3M} || \theta_{dj}^* - \theta_{dj} ||)$. It indicates the similar accumulative effect which leads to fast growth of the magnitude regarding the model depth (see \cref{postln_model_update}). Furthermore, the cross-attention propagates the magnitude from the encoder to the decoder, which explains why the decoder is more unstable than the encoder~\citep{LiyuanLiu2020UnderstandingTD}.

\begin{figure*}[t]
\begin{center}
\includegraphics[width=0.4\columnwidth]{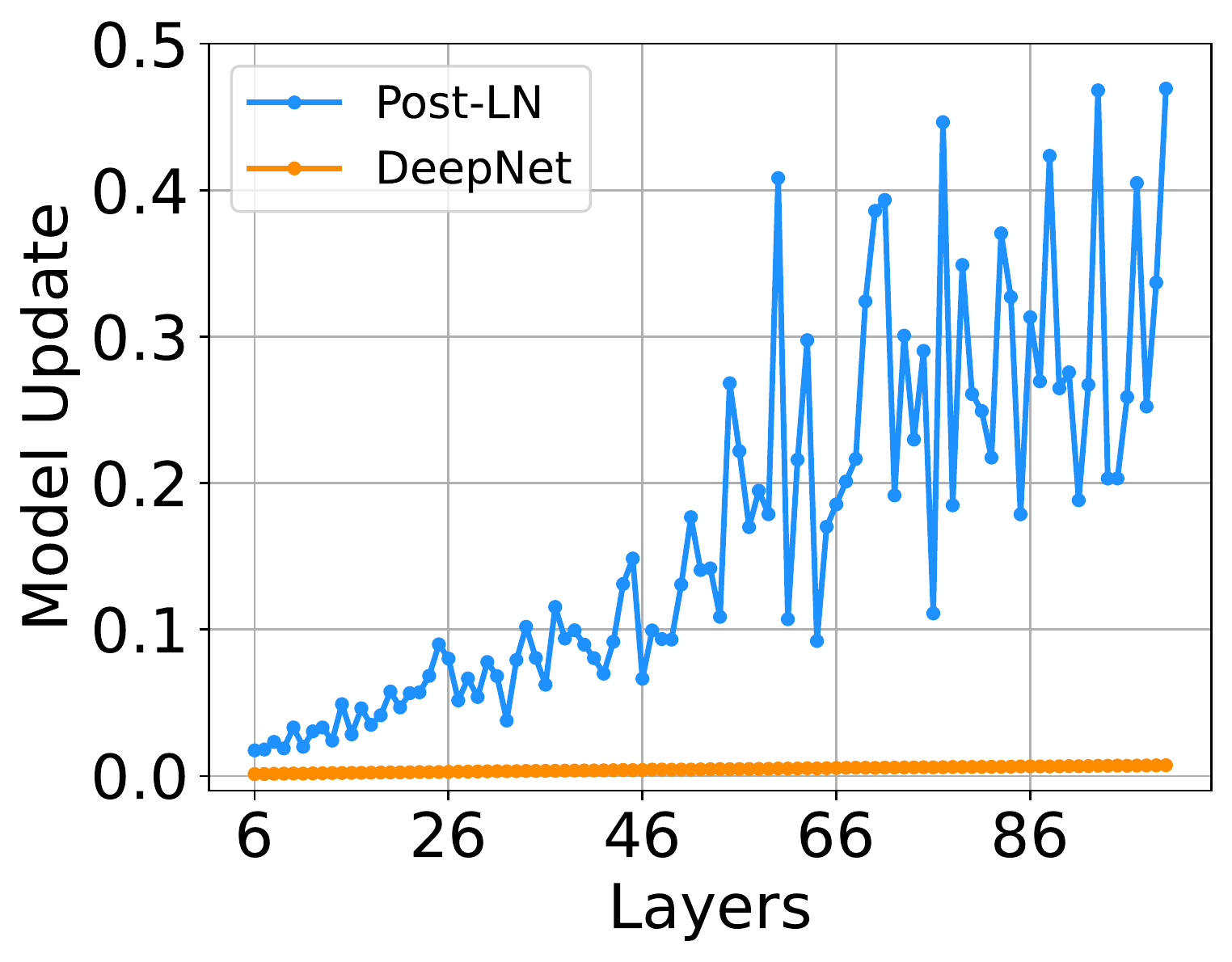}
\caption{Model updates of vanilla \postln{} and \our{} at the early stage of training. The visualization is conducted on 64-128-2 tiny Transformers with depth varying from 6L-6L to 100L-100L. It shows that \our{} has much smaller and more stable updates than \postln{}.}\label{postln_model_update}
\end{center}
\end{figure*}

\subsection{Derivation for \deepnorm{} and the Initialization}
\label{set:imple}

We show that the expected model updates for \our{} can be bounded by a constant with proper parameters $\alpha$ and $\beta$.
Our analysis is based on SGD update, and we empirically verify it works well for Adam optimizer~\citep{adam}. 
We provide the analysis on the encoder-decoder architecture, which can be naturally extended to encoder-only and decoder-only models in the same way.
Analogous to \citet{HongyiZhang2019FixupIR}, we set our goal for the model update as follows:

\begin{framed}
\textbf{GOAL:} $F_{ed}(x, y, \theta_e, \theta_d )$ is updated by $\Theta(\eta)$ per SGD step after initialization as $\eta \rightarrow 0$. That is $|| \Delta F_{ed} || = \Theta(\eta)$ where $\Delta F_{ed} \overset{\Delta}{=} F_{ed}(x, y, \theta_e - \eta \frac{\partial \mathcal{L}}{\partial \theta_e}, \theta_d - \eta \frac{\partial \mathcal{L}}{\partial \theta_d}) - F_{ed}(x, y, \theta_e, \theta_d )$.      
\end{framed}

For SGD optimizer, the update of each decoder layer $|| \theta_{di}^* - \theta_{di} ||$ equals to $\eta || \frac{\partial \mathcal{L}}{\partial \theta_{di}} ||$. \cite{RuibinXiong2020OnLN} proved that \postln{} decreases the magnitude of backpropagating error signal, so we have $ || \frac{\partial F}{\partial \theta_{dj}} || \le || \frac{\partial F}{\partial \theta_{d,3M}} ||$. With $|| \frac{\partial F}{\partial \theta_{d,3M}} || \overset{\Theta}{=} \frac{||\theta_{d,3M}||}{\alpha_d}$ and the assumption $||\frac{\partial \mathcal{L}}{\partial F} || = \mathcal{O}(1)$, the second term of \cref{eq: en2de} can be bounded as:

\begin{align}
    \sum_{j=1}^{3M} \cfrac{\sqrt{v_{dj}^2 + w_{dj}^2}}{\alpha_{d}} ||\theta_{dj}^* - \theta_{dj}|| 
    &\le \eta ||\cfrac{\partial \mathcal{L}}{\partial F} ||\cdot ||\cfrac{\partial \mathcal{F}}{\partial \theta_{d, 3M}}|| \sum_{j=1}^{3M} \cfrac{\sqrt{v_{dj}^2 + w_{dj}^2}}{\alpha_d} \notag \\
    &\overset{\Theta}{=} 3\eta M \cfrac{v_d^2 + w_d^2}{\alpha_d^2} \label{eq:bound-impl}
\end{align}

There are multiple schemes to bound \cref{eq:bound-impl} by $\Theta(\eta)$.
In order to balance the effect of residual connections and the initialization, we set $\alpha_d^2 = (3M)^{\frac{1}{2}}$, $v_d^2 + w^d_2 = (3M)^{\frac{1}{2}}$ and $v_d = w_d = \beta_d$ due to symmetry, that is $\alpha_d = (3M)^{\frac{1}{4}}$, $\beta_d = (12M)^{-\frac{1}{4}}$. Similarly, we use $v_e = w_e = \beta_e = 0.87(N^4 M)^{-\frac{1}{16}}, \alpha_e = 0.81(N^4 M)^{\frac{1}{16}}$ to bound the first term in \cref{eq: en2de}. Detailed derivation is shown in \cref{appendix_derivation_en2de}.

In comparison with \postln{}, we visualize the model updates for \our{} on IWSLT-14 De-En translation dataset at the early training stage.
\cref{postln_model_update} shows that the model update of \our{} is nearly constant, while the model update of \postln{} is exploding.

In summary, we apply our approach as follows:

\begin{tcolorbox}[enhanced,attach boxed title to top center={yshift=-3mm,yshifttext=-1mm}, title=\textbf{Encoder-decoder architecture}, colback=white, colframe=white!75!blue, coltitle=black, colbacktitle=white]
    \begin{enumerate}[leftmargin=*]
    \item Apply standard initialization (e.g., Xavier initialization) for each encoder and decoder layer.
    \item For encoder layers, scale the weights of feed-forward networks as well as the value projection and the output projection of attention layers by $0.87(N^4 M)^{-\frac{1}{16}}$, and set the weight of residual connections as $0.81(N^4 M)^{\frac{1}{16}}$.
    \item For decoder layers, scale the weights of feed-forward networks as well as the value projection and the output projection of attention layers by $(12M)^{-\frac{1}{4}}$, and set the weight of residual connections as $(3M)^{\frac{1}{4}}$.
    \end{enumerate}
\end{tcolorbox}

The derivation of encoder-only (such as BERT) and decoder-only (such as GPT) architectures can be conducted in the same way (see \cref{appendix_derivation}). We summarize the steps as follows:

\begin{tcolorbox}[enhanced,attach boxed title to top center={yshift=-3mm,yshifttext=-1mm}, title=\textbf{Encoder-only (or decoder-only) architecture}, colback=white, colframe=white!75!blue, coltitle=black, colbacktitle=white]
    \begin{enumerate}[leftmargin=*]
    \item Apply standard initialization (e.g., Xavier initialization) for each layer.
    \item For each layer, scale the weights of feed-forward networks as well as the value projection and the output projection of attention layers by $(8N)^{-\frac{1}{4}}$ (or $(8M)^{-\frac{1}{4}}$), and set the weight of residual connections as $(2N)^{\frac{1}{4}}$ (or $(2M)^{\frac{1}{4}}$).
    \end{enumerate}
\end{tcolorbox}

\section{Neural Machine Translation}
\label{nmt}

\begin{table*}[t]
\begin{center}
\begin{tabular}{l|c|cccc}
\toprule
\textbf{Models} & \textbf{LN}  & \textbf{6L-6L} & \textbf{18L-18L} & \textbf{50L-50L} &
 \textbf{100L-100L} \\
\midrule
Vanilla \postln{}~\citep{transformer} & Post & \textbf{28.1}	&  \multicolumn{3}{c}{diverged} \\
 DS-Init~\citep{BiaoZhang2019ImprovingDT} & Post & 27.9 & \multicolumn{3}{c}{diverged}  \\
 Admin~\citep{LiyuanLiu2020UnderstandingTD} & Post & 27.9 & \textbf{28.8} &  \multicolumn{2}{c}{diverged} \\
 \midrule
 ReZero~\citep{rezero2020} & No & 26.9 & \multicolumn{3}{c}{diverged} \\
 R-Fixup~\citep{HongyiZhang2019FixupIR} & No & 27.5 & 28.4 & 27.7 & diverged \\
 T-Fixup~\citep{XiaoShiHuang2020ImprovingTO} & No & 27.5 & 28.4 & 27.9	& diverged \\
 \midrule
 Vanilla \preln{}~\citep{transformer} & Pre & 27.0	& 28.1	& 28.0 & 27.4 \\
  DLCL~\citep{Wang2019DLCL} & Pre & 27.4 & 28.2 & diverged & 27.5
 \\
  NormFormer~\citep{Normformer2021} & Pre & 27.0 & 28.3 & 27.8 & diverged
 \\
 \midrule
 
 \bf \our{} (ours) & Deep & 27.8 & \textbf{28.8} & \textbf{29.0} & \textbf{28.9}
 \\
\bottomrule
\end{tabular}
\caption{BLEU scores on the WMT-17 En-De test set for different models with varying depth. $A$L-$B$L refers to $A$-layer encoder and $B$-layer decoder.}
\label{tab:wmt17}
\end{center}
\end{table*}

We verify the effectiveness of \our{} on the popular machine translation benchmarks, including IWSLT-14 German-English (De-En) dataset and WMT-17 English-German (En-De) dataset. We compare our method with multiple state-of-the-art deep Transformer models, including DLCL~\citep{Wang2019DLCL}, NormFormer~\citep{Normformer2021}, ReZero~\citep{rezero2020}, R-Fixup~\citep{HongyiZhang2019FixupIR}, T-Fixup~\citep{XiaoShiHuang2020ImprovingTO}, DS-init~\citep{BiaoZhang2019ImprovingDT}, and Admin~\citep{LiyuanLiu2020UnderstandingTD}. We reproduce the baselines with their open-source code, and set the hyper-parameters the same for a fair comparison.

We use BLEU as the evaluation metric for all experiments. \cref{tab:wmt17} reports the results of the baselines and \our{} on WMT-17 En-De translation dataset. According to their LNs, the baselines are grouped into three categories: \preln{}, \postln{}, and No-LN. All the compared models are base-size with different depths.

Compared with the models with \postln{}, \our{} is more stable, and can successfully scale to 100L-100L, reaching the 28.9 BLEU on the test set. In contrast, the baselines with \postln{} lead to unstable optimization when the depth goes to 50L-50L. Besides, \our{} achieves comparable performance with these baselines when the models are shallow.

In addition, we compare \our{} with the methods without LN. Both R-Fixup and T-Fixup introduce better initialization methods, which stabilize the training of No-LN Transformer with up to 50-50 layers. Yet, their performance is not as good as those with \postln{}. Besides, half-precision could destabilize the training of ReZero, leading to its divergence with 18-18 layers. This observation is also reported by \citet{LiyuanLiu2020UnderstandingTD}. Moreover, deeper models (50L-50L) do not outperform the shallow models (18L-18L). In comparison, \our{} achieves better translation accuracy than these methods, and scaling to deeper models brings no harm to the performance.

Compared with the \postln{} baselines, the models with \preln{} are more stable. Both vanilla \preln{} and DLCL can be scaled to 100L-100L, and 50L-50L NormFormer is also trained successfully. Nevertheless, \preln{} leads to a 0.5-1.0 BLEU drop compared with the converged \postln{} models. We presume this should be caused by the problem that gradients of \preln{} at earlier layers tend to be larger than gradients at later layers~\citep{Normformer2021}. We leave it as the future work. In contrast, \our{} alleviates the problem by using \postln{}, and outperforms all the \preln{} baselines.

\paragraph{Convergence with varying depth.}
We vary the depths of the models from 10L-10L to 100L-100L with an interval of 10 layers.All experiments are conducted with mixed precision training, except ReZero\footnote{According to our experiments, ReZero is unstable with half precision, even when the model is shallow.}. \cref{iwslt_convergence} shows the results on the IWSLT-14 dataset. We train the models for 8,000 steps because we find most divergence occurs at the beginning of optimization. Overall, \our{} is stable from shallow to deep. It converges fast, achieving over 30 BLEU in only 8,000 steps while most of the baselines do not. Moreover, the performance keeps improving as the model goes deeper. 

\paragraph{Large learning rate, batch size, and hidden dimension.}
We further scale \our{} to larger learning rate, batch size, and hidden dimension, respectively. For each experiment, we only change one hyperparameter with the others fixed. \cref{large_hyper} reports the loss curves on the WMT-17 validation set. It shows that \our{} can be trained without difficulty in all the largest settings. The loss of \our{} with 1024 hidden size increases after 10K steps because of overfitting. Besides, it indicates that \our{} can benefit from the larger settings, resulting in faster convergence and lower validation loss.

\begin{figure}[t]
\begin{center}
{
\includegraphics[width=0.31\columnwidth]{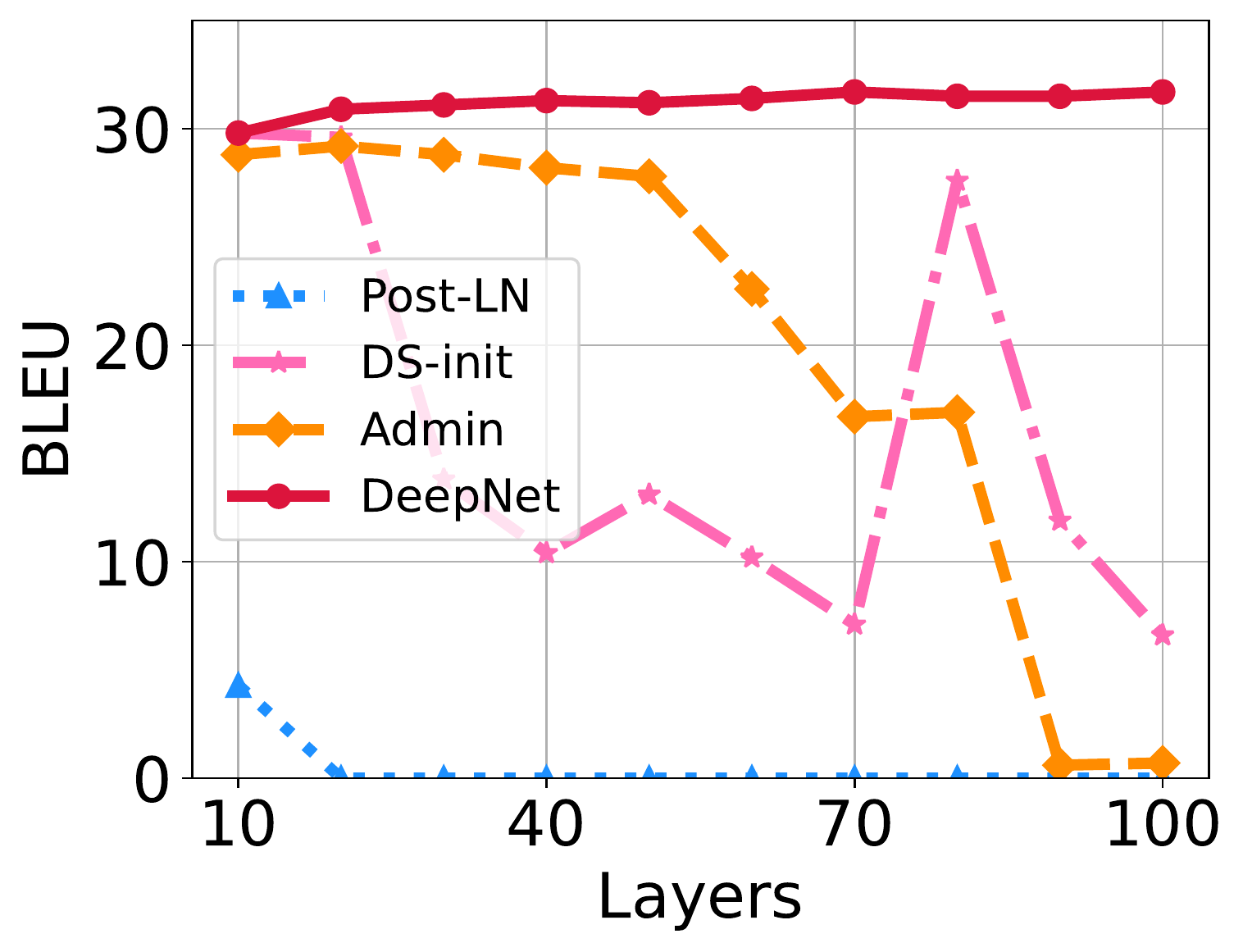}
\label{iwslt_post}
}
{
\includegraphics[width=0.31\columnwidth]{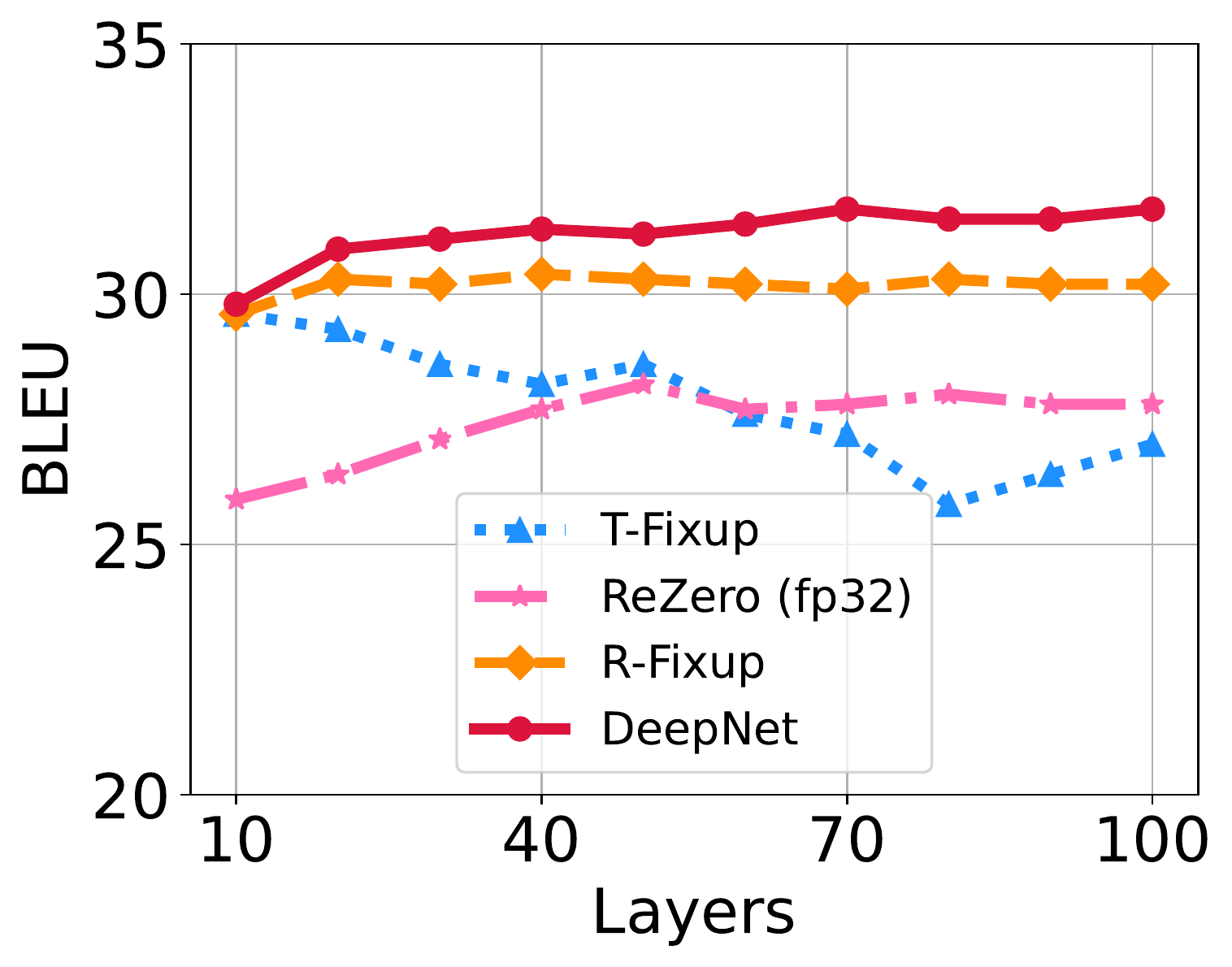}
\label{iwslt_no_norm}
}
{
\includegraphics[width=0.31\columnwidth]{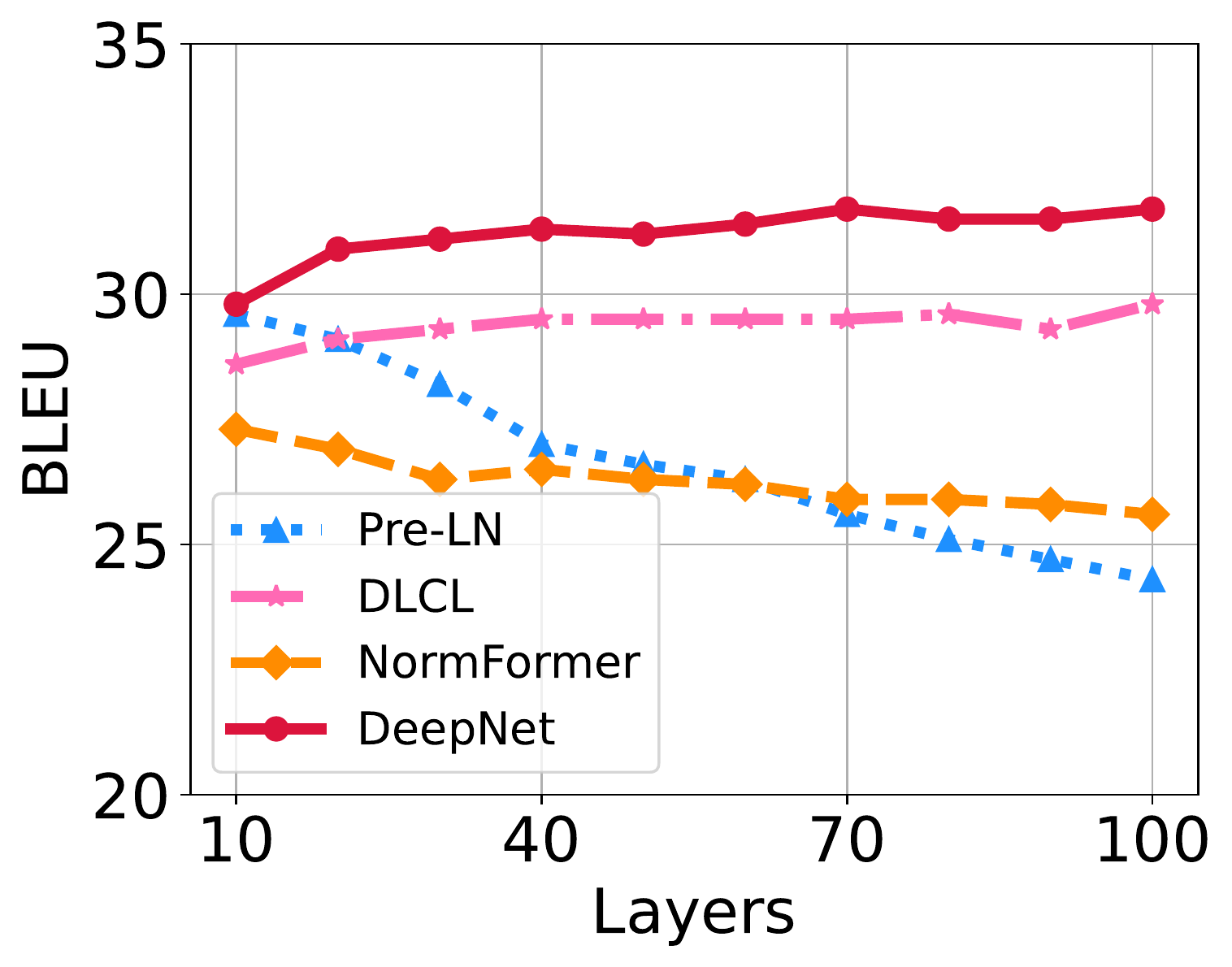}
\label{iwslt_pre}
}

\caption{BLEU scores on the IWSLT-14 De-En test set for different deep models with varing depth from 10L-10L to 100L-100L.}\label{iwslt_convergence}
\end{center}
\end{figure}

\begin{figure}[t]
\begin{center}
{
\includegraphics[width=0.31\columnwidth]{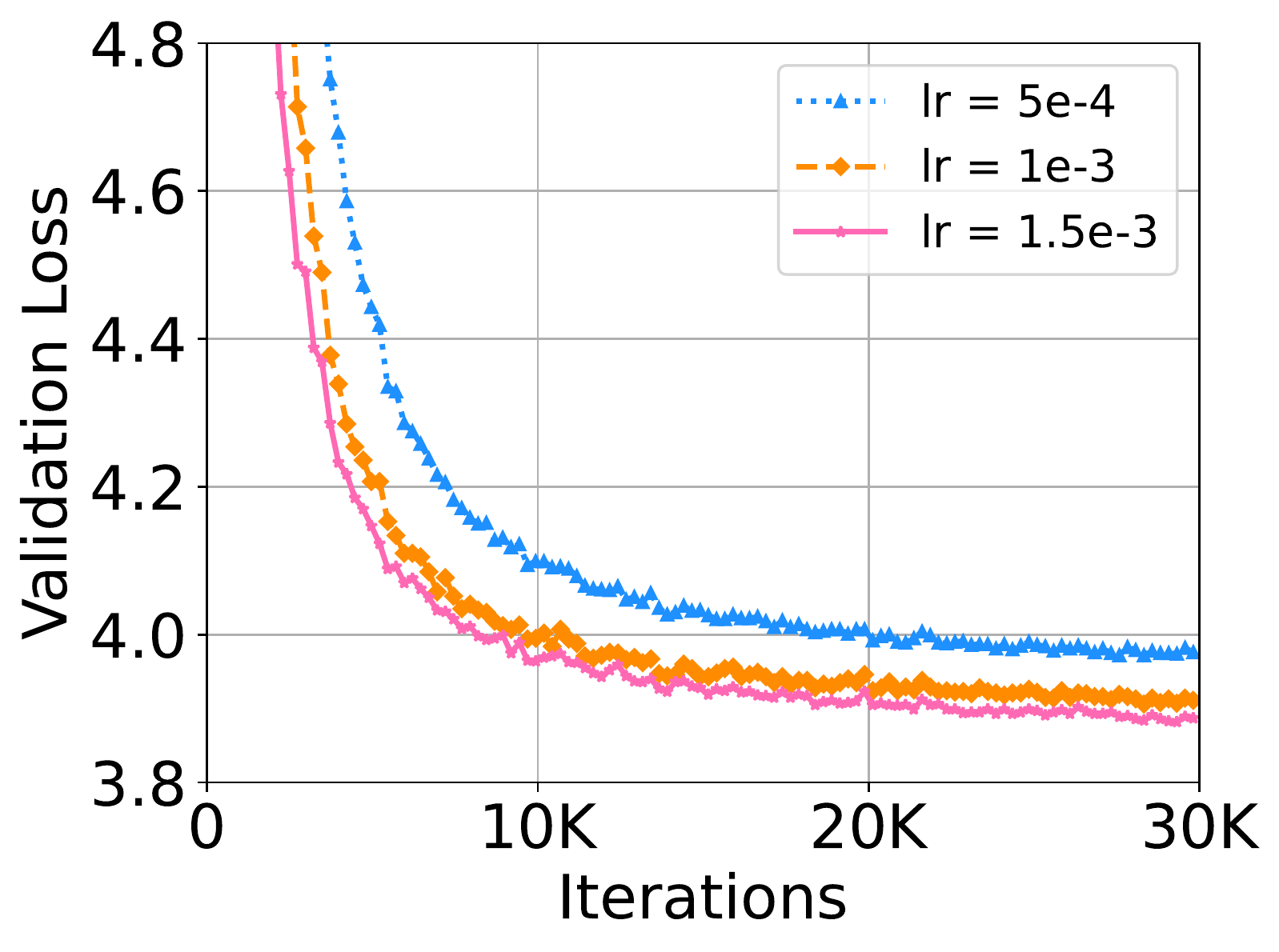}
\label{valid_loss_lr}
}
{
\includegraphics[width=0.31\columnwidth]{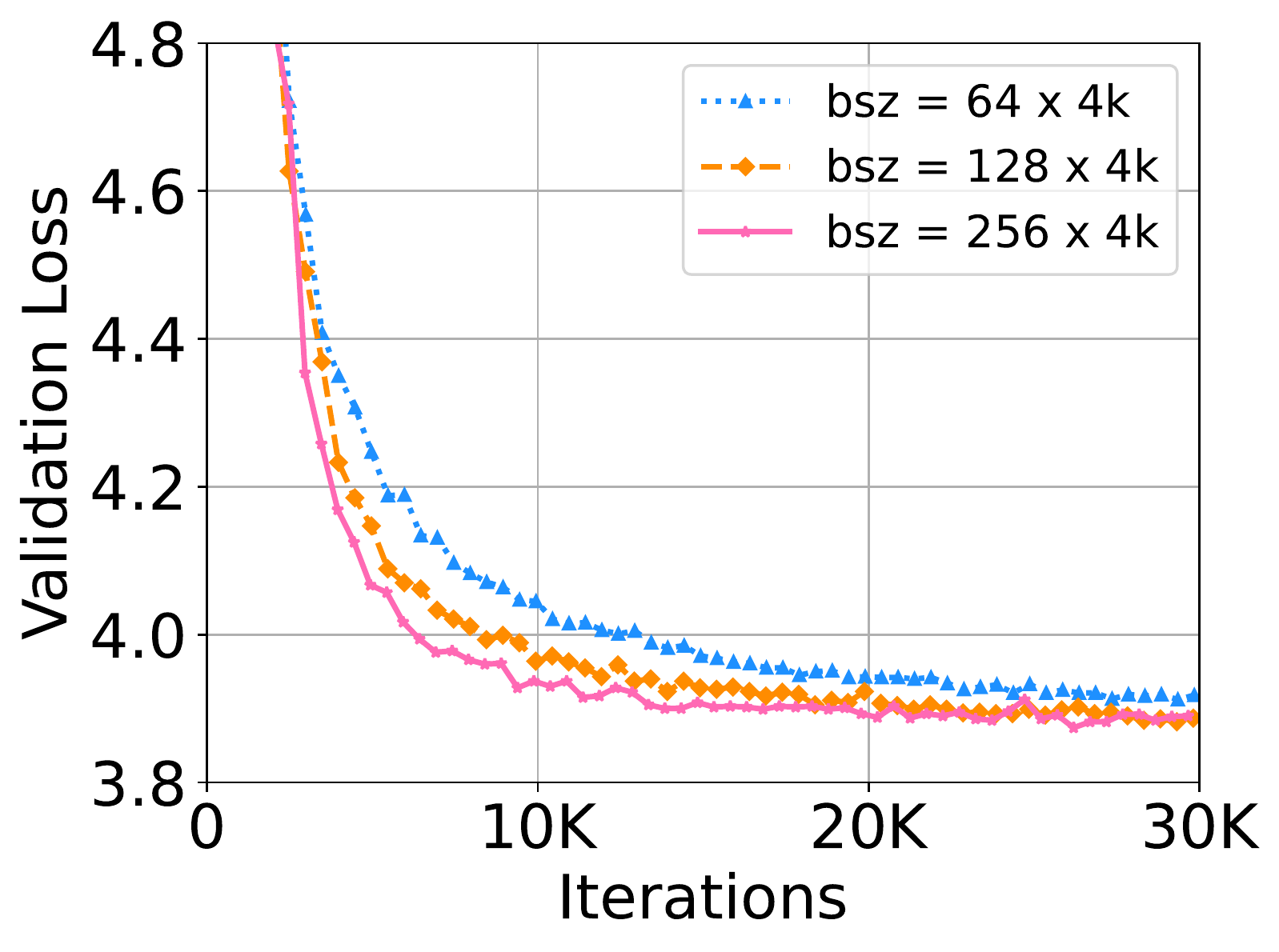}
\label{valid_loss_bsz}
}
{
\includegraphics[width=0.31\columnwidth]{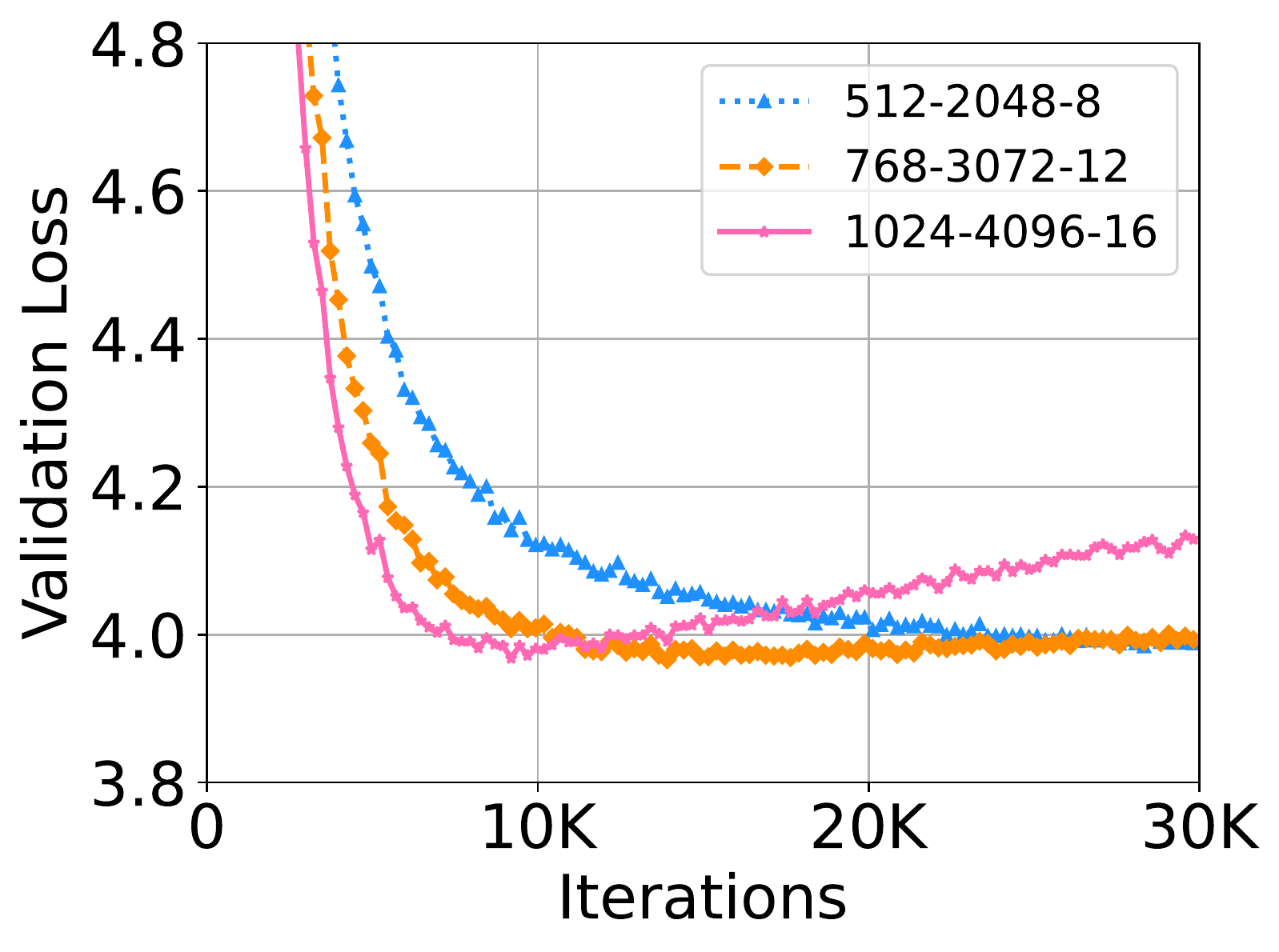}
\label{valid_loss_hidden}
}

\caption{WMT-17 En-De validation loss curves for 18L-18L \our{} with varing learning rate, batch size and hidden dimension.}
\label{large_hyper}
\end{center}
\end{figure}

\section{Massively Multilingual Neural Machine Translation}
\label{mnmt}

\begin{figure*}[t]
  \centering
  \includegraphics[width=0.4\columnwidth]{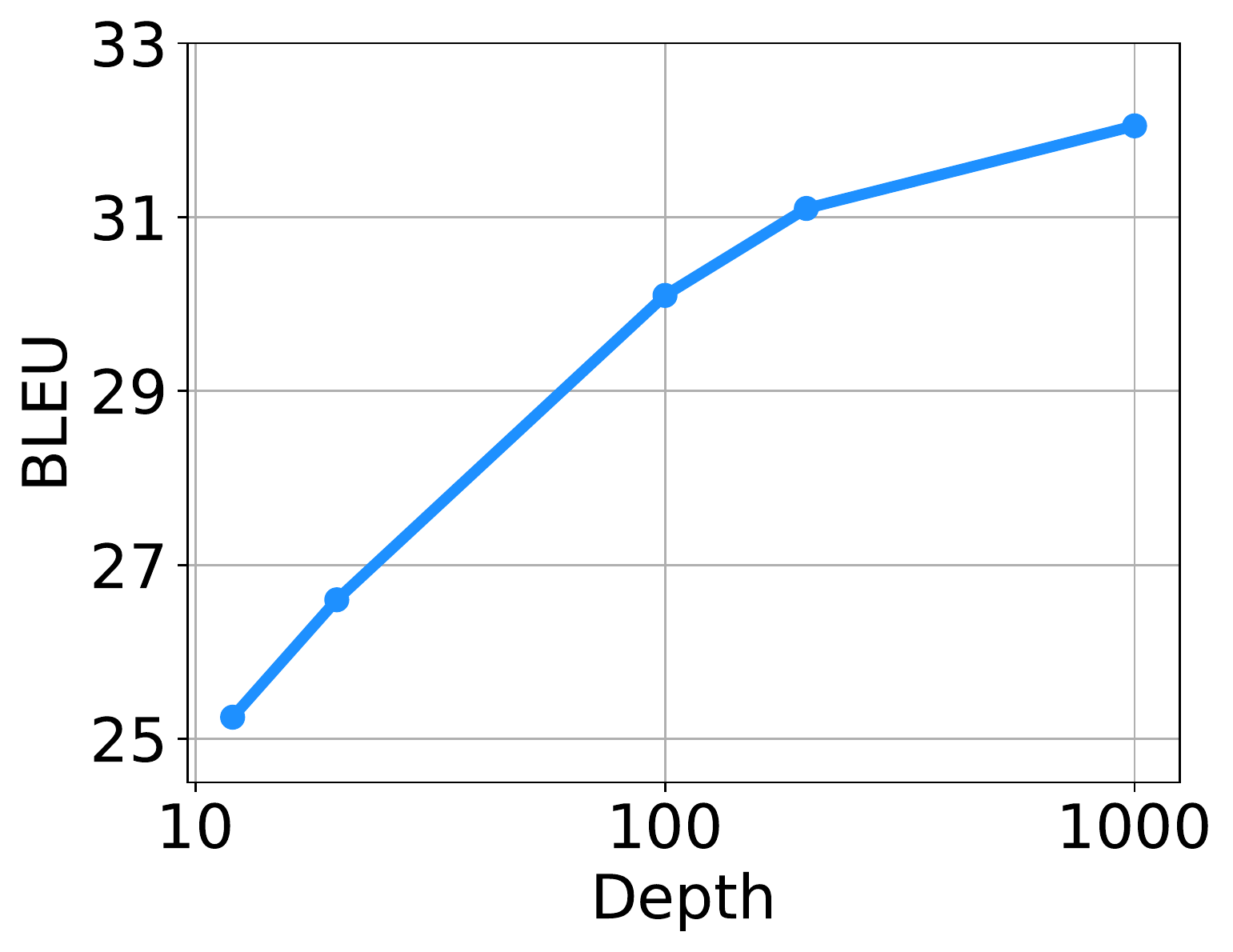}
  \caption{Average BLEU scores for \our{} with varying depth on the OPUS-100 En-X and X-En test sets.}\label{opus_curve}
\end{figure*}

\begin{table*}[t]
\begin{center}
\begin{tabular}{l|cc|cccc}
\toprule
\textbf{Models} & \textbf{\# Layers}  & \textbf{\# Params} & \textbf{X$\rightarrow$En} & \textbf{En$\rightarrow$X} & \textbf{Avg} \\
\midrule
\multirow{3}{*}{Baseline~\citep{opus100}} & 12 & 133M & 27.5 & 21.4 & 24.5 \\
 & 24 & 173M & 29.5 & 22.9 & 26.2 \\
 & 48 & 254M & 31.4 & 24.0 & 27.7 \\
\midrule
\multirow{2}{*}{\textbf{\our{} (ours)}} & 200 & 863M & 33.2 & 29.0 & 31.1 \\
 & 1000 & 3.8B & \bf 33.9 & \bf 30.2 & \bf 32.1 \\
\bottomrule
\end{tabular}
\caption{Average BLEU for \our{} and the baseline on the OPUS-100 test sets.}
\label{opus}
\end{center}
\end{table*}

\begin{table*}[t]
\begin{center}
\begin{tabular}{l|cc|ccccc}
\toprule
\textbf{Models} & \textbf{\# Layers}  & \textbf{\# Params} & \textbf{WMT} & \textbf{OPUS} & \textbf{TED} & \textbf{Flores} \\
\midrule
M2M-100~\citep{m2m100} & 48 & 12B & 31.9 & 18.4 & 18.7 & 13.6 \\
\textbf{\our{} (ours)} & 200 & 3.2B & \bf 33.9 & \bf 23.0 & \bf 20.1 & \bf 18.6 \\
\bottomrule
\end{tabular}
\caption{BLEU scores for \our{} and M2M-100 on various evaluation sets.}
\label{m2m}
\end{center}
\end{table*}

We conduct experiments on the large-scale multilingual machine translation, which is a good testbed for large models. We first use OPUS-100 corpus~\citep{opus100} to evaluate our model. OPUS-100 is an English-centric multilingual corpus covering 100 languages, which is randomly sampled from the OPUS collection. We scale \our{} up to 1,000 layers. The model has a 500-layer encoder, a 500-layer decoder, 512 hidden size, 8 attention head, and 2,048 dimensions of feed-forward layers. More details can be found in the Appendix.

\cref{opus} summarizes the results of \our{} and the baselines. It shows that increasing the depth can significantly improve the translation quality of NMT: the baseline of 48 layers achieves a gain of 3.2 points on average over the 12-layer model. \our{} can successfully scale up the depth to 1,000 layers, outperforming the baseline by an improvement of 4.4 BLEU. It is noted that \our{} is only trained for 4 epochs, and the performance can be further improved given more computation budgets.

\paragraph{Scaling law in terms of depth} 

We train \our{} of \{12, 20, 100, 200, 1000\} layers on the OPUS-100 dataset. \cref{opus_curve} illustrates the scaling curve. Compared with bilingual NMT, multilingual NMT benefits more from scaling the depth of the model because of its hunger in model capacity. We observe logarithmic growth of the BLEU score for multilingual NMT, and the scaling law can be written as:
\begin{equation*}
    L(d)=A\log(d)+B
\end{equation*}
where $d$ is the depth, and $A,B$ are the constants regarding the other hyper-parameters.

\paragraph{More data and language directions.}
To explore the limits of \our{} on multilingual NMT, we then scale up the training data by using CCMatrix~\citep{ccmatrix}. We also expand the data from CCAligned~\citep{ccaligned}, OPUS~\citep{opus100}, and Tatoeba\footnote{\url{https://tatoeba.org/en/}} to cover all languages of Flores101 evaluation sets. The final data consists of 102 languages, 1932 directions, and 12B sentence pairs. With the data, we train \our{} with a 100-layer encoder, 100-layer decoder, 1,024 hidden dimension, 16 heads, and 4,096 intermediate dimension of feed-forward layers. More details can be found in the Appendix.

We compare \our{} with the state-of-the-art multilingual NMT model M2M-100~\citep{m2m100}. M2M-100 has a 24-layer encoder, a 24-layer decoder, and 4,096 hidden size, resulting in up to 12B parameters. Compared with M2M-100, \our{} is deep and narrow with only 3.2B parameters. For a fair comparison, we generate the model with beam size 5 and length penalty 1.

Following M2M-100~\citep{m2m100}, we evaluate the models on several multilingual translation evaluation datasets, including WMT~\citep{wmt14,wmt17,wmt18,wmt19}, OPUS~\citep{opus100}, TED~\citep{ted-data}, and Flores~\citep{flores101}. The language pairs from the WMT dataset are English-centric. There are 10 languages including English, and most of them are high-resource. For the OPUS dataset, we select the non-English directions from the test set, which has 30 evaluation pairs. The TED evaluation set has 28 languages and 756 directions, and the data is from the spoken language domain. The Flores dataset has all translation pairs between 102 languages. We use a subset covering the languages supported by both M2M-100 and \our{}, resulting in 87 languages and 7,482 translation directions.

We report the results in~\cref{m2m}. For a fair comparison, we use the same evaluation methods as the baseline. The details can be found in the Appendix. It shows that \our{} has significantly better performance than M2M-100 on all evaluation datasets, indicating that deepening the model is a very promising direction to improve the quality of NMT models.

\section{Conclusion and Future Work}

We improve the stability of Transformer and successfully scale it to 1,000 layers. This is achieved by our \our{} with a novel normalization function called \deepnorm{}. It has theoretical justification to stabilize the optimization with a constant upper bound for model updates. Experimental results verify the effectiveness of our methods across various benchmarks.
We focus on machine translation as a test bed in the current experiments.
In the future, we will extend \our{} to support more diverse tasks, e.g., language model pre-training~\citep{unilm,unilmv2,chi-etal-2021-infoxlm,deltalm,chi2021xlme}, protein structure prediction~\citep{AlphaFold2021}, and BEiT vision pre-training~\citep{beit,vlmo}.

\paragraph{Acknowledgement}
We would like to acknowledge Saksham Singhal for the CCMatrix corpus.

\bibliographystyle{plainnat}
\bibliography{deepnet}

\begin{thebibliography}{44}
\providecommand{\natexlab}[1]{#1}
\providecommand{\url}[1]{\texttt{#1}}
\expandafter\ifx\csname urlstyle\endcsname\relax
  \providecommand{\doi}[1]{doi: #1}\else
  \providecommand{\doi}{doi: \begingroup \urlstyle{rm}\Url}\fi

\bibitem[Bachlechner et~al.(2020)Bachlechner, Majumder, Mao, Cottrell, and
  McAuley]{rezero2020}
Thomas Bachlechner, Bodhisattwa~Prasad Majumder, Huanru~Henry Mao, Garrison~W.
  Cottrell, and Julian~J. McAuley.
\newblock Rezero is all you need: Fast convergence at large depth.
\newblock \emph{CoRR}, abs/2003.04887, 2020.

\bibitem[Bao et~al.(2020)Bao, Dong, Wei, Wang, Yang, Liu, Wang, Piao, Gao,
  Zhou, and Hon]{unilmv2}
Hangbo Bao, Li~Dong, Furu Wei, Wenhui Wang, Nan Yang, Xiaodong Liu, Yu~Wang,
  Songhao Piao, Jianfeng Gao, Ming Zhou, and Hsiao-Wuen Hon.
\newblock {UniLMv2}: Pseudo-masked language models for unified language model
  pre-training.
\newblock In \emph{{ICML} 2020}, volume 119 of \emph{Proceedings of Machine
  Learning Research}, pages 642--652. PMLR, 2020.

\bibitem[Bao et~al.(2022)Bao, Dong, Piao, and Wei]{beit}
Hangbo Bao, Li~Dong, Songhao Piao, and Furu Wei.
\newblock {BEiT}: {BERT} pre-training of image transformers.
\newblock In \emph{International Conference on Learning Representations}, 2022.

\bibitem[Barrault et~al.(2019)Barrault, Bojar, Costa{-}juss{\`{a}}, Federmann,
  Fishel, Graham, Haddow, Huck, Koehn, Malmasi, Monz, M{\"{u}}ller, Pal, Post,
  and Zampieri]{wmt19}
Lo{\"{\i}}c Barrault, Ondrej Bojar, Marta~R. Costa{-}juss{\`{a}}, Christian
  Federmann, Mark Fishel, Yvette Graham, Barry Haddow, Matthias Huck, Philipp
  Koehn, Shervin Malmasi, Christof Monz, Mathias M{\"{u}}ller, Santanu Pal,
  Matt Post, and Marcos Zampieri.
\newblock Findings of the 2019 conference on machine translation {(WMT19)}.
\newblock In Ondrej Bojar, Rajen Chatterjee, Christian Federmann, Mark Fishel,
  Yvette Graham, Barry Haddow, Matthias Huck, Antonio Jimeno{-}Yepes, Philipp
  Koehn, Andr{\'{e}} Martins, Christof Monz, Matteo Negri, Aur{\'{e}}lie
  N{\'{e}}v{\'{e}}ol, Mariana~L. Neves, Matt Post, Marco Turchi, and Karin
  Verspoor, editors, \emph{Proceedings of the Fourth Conference on Machine
  Translation, {WMT} 2019, Florence, Italy, August 1-2, 2019 - Volume 2: Shared
  Task Papers, Day 1}, pages 1--61. Association for Computational Linguistics,
  2019.

\bibitem[Bojar et~al.(2014)Bojar, Buck, Federmann, Haddow, Koehn, Leveling,
  Monz, Pecina, Post, Saint{-}Amand, Soricut, Specia, and Tamchyna]{wmt14}
Ondrej Bojar, Christian Buck, Christian Federmann, Barry Haddow, Philipp Koehn,
  Johannes Leveling, Christof Monz, Pavel Pecina, Matt Post, Herve
  Saint{-}Amand, Radu Soricut, Lucia Specia, and Ales Tamchyna.
\newblock Findings of the 2014 workshop on statistical machine translation.
\newblock In \emph{Proceedings of the Ninth Workshop on Statistical Machine
  Translation, WMT@ACL 2014, June 26-27, 2014, Baltimore, Maryland, {USA}},
  pages 12--58. The Association for Computer Linguistics, 2014.

\bibitem[Bojar et~al.(2017)Bojar, Chatterjee, Federmann, Graham, Haddow, Huang,
  Huck, Koehn, Liu, Logacheva, Monz, Negri, Post, Rubino, Specia, and
  Turchi]{wmt17}
Ondrej Bojar, Rajen Chatterjee, Christian Federmann, Yvette Graham, Barry
  Haddow, Shujian Huang, Matthias Huck, Philipp Koehn, Qun Liu, Varvara
  Logacheva, Christof Monz, Matteo Negri, Matt Post, Raphael Rubino, Lucia
  Specia, and Marco Turchi.
\newblock Findings of the 2017 conference on machine translation {(WMT17)}.
\newblock In Ondrej Bojar, Christian Buck, Rajen Chatterjee, Christian
  Federmann, Yvette Graham, Barry Haddow, Matthias Huck, Antonio
  Jimeno{-}Yepes, Philipp Koehn, and Julia Kreutzer, editors, \emph{Proceedings
  of the Second Conference on Machine Translation, {WMT} 2017, Copenhagen,
  Denmark, September 7-8, 2017}, pages 169--214. Association for Computational
  Linguistics, 2017.

\bibitem[Bojar et~al.(2018)Bojar, Federmann, Fishel, Graham, Haddow, Koehn, and
  Monz]{wmt18}
Ondrej Bojar, Christian Federmann, Mark Fishel, Yvette Graham, Barry Haddow,
  Philipp Koehn, and Christof Monz.
\newblock Findings of the 2018 conference on machine translation {(WMT18)}.
\newblock In Ondrej Bojar, Rajen Chatterjee, Christian Federmann, Mark Fishel,
  Yvette Graham, Barry Haddow, Matthias Huck, Antonio Jimeno{-}Yepes, Philipp
  Koehn, Christof Monz, Matteo Negri, Aur{\'{e}}lie N{\'{e}}v{\'{e}}ol,
  Mariana~L. Neves, Matt Post, Lucia Specia, Marco Turchi, and Karin Verspoor,
  editors, \emph{Proceedings of the Third Conference on Machine Translation:
  Shared Task Papers, {WMT} 2018, Belgium, Brussels, October 31 - November 1,
  2018}, pages 272--303. Association for Computational Linguistics, 2018.

\bibitem[Brown et~al.(2020)Brown, Mann, Ryder, Subbiah, Kaplan, Dhariwal,
  Neelakantan, Shyam, Sastry, Askell, Agarwal, Herbert{-}Voss, Krueger,
  Henighan, Child, Ramesh, Ziegler, Wu, Winter, Hesse, Chen, Sigler, Litwin,
  Gray, Chess, Clark, Berner, McCandlish, Radford, Sutskever, and Amodei]{gpt3}
Tom~B. Brown, Benjamin Mann, Nick Ryder, Melanie Subbiah, Jared Kaplan,
  Prafulla Dhariwal, Arvind Neelakantan, Pranav Shyam, Girish Sastry, Amanda
  Askell, Sandhini Agarwal, Ariel Herbert{-}Voss, Gretchen Krueger, Tom
  Henighan, Rewon Child, Aditya Ramesh, Daniel~M. Ziegler, Jeffrey Wu, Clemens
  Winter, Christopher Hesse, Mark Chen, Eric Sigler, Mateusz Litwin, Scott
  Gray, Benjamin Chess, Jack Clark, Christopher Berner, Sam McCandlish, Alec
  Radford, Ilya Sutskever, and Dario Amodei.
\newblock Language models are few-shot learners.
\newblock In \emph{NeurIPS 2020}, 2020.

\bibitem[Chi et~al.(2021{\natexlab{a}})Chi, Dong, Wei, Yang, Singhal, Wang,
  Song, Mao, Huang, and Zhou]{chi-etal-2021-infoxlm}
Zewen Chi, Li~Dong, Furu Wei, Nan Yang, Saksham Singhal, Wenhui Wang, Xia Song,
  Xian{-}Ling Mao, Heyan Huang, and Ming Zhou.
\newblock {InfoXLM}: An information-theoretic framework for cross-lingual
  language model pre-training.
\newblock In \emph{{NAACL-HLT} 2021}, pages 3576--3588, 2021{\natexlab{a}}.

\bibitem[Chi et~al.(2021{\natexlab{b}})Chi, Huang, Dong, Ma, Singhal, Bajaj,
  Song, and Wei]{chi2021xlme}
Zewen Chi, Shaohan Huang, Li~Dong, Shuming Ma, Saksham Singhal, Payal Bajaj,
  Xia Song, and Furu Wei.
\newblock {XLM-E:} cross-lingual language model pre-training via {ELECTRA}.
\newblock \emph{CoRR}, abs/2106.16138, 2021{\natexlab{b}}.

\bibitem[Conneau et~al.(2020)Conneau, Khandelwal, Goyal, Chaudhary, Wenzek,
  Guzm{\'{a}}n, Grave, Ott, Zettlemoyer, and Stoyanov]{xlmr}
Alexis Conneau, Kartikay Khandelwal, Naman Goyal, Vishrav Chaudhary, Guillaume
  Wenzek, Francisco Guzm{\'{a}}n, Edouard Grave, Myle Ott, Luke Zettlemoyer,
  and Veselin Stoyanov.
\newblock Unsupervised cross-lingual representation learning at scale.
\newblock In Dan Jurafsky, Joyce Chai, Natalie Schluter, and Joel~R. Tetreault,
  editors, \emph{{ACL} 2020}, pages 8440--8451, 2020.

\bibitem[Devlin et~al.(2019)Devlin, Chang, Lee, and
  Toutanova]{JacobDevlin2018BERTPO}
Jacob Devlin, Ming{-}Wei Chang, Kenton Lee, and Kristina Toutanova.
\newblock {BERT}: Pre-training of deep bidirectional transformers for language
  understanding.
\newblock In \emph{{NAACL-HLT} 2019}, pages 4171--4186, 2019.

\bibitem[Dong et~al.(2019)Dong, Yang, Wang, Wei, Liu, Wang, Gao, Zhou, and
  Hon]{unilm}
Li~Dong, Nan Yang, Wenhui Wang, Furu Wei, Xiaodong Liu, Yu~Wang, Jianfeng Gao,
  Ming Zhou, and Hsiao{-}Wuen Hon.
\newblock Unified language model pre-training for natural language
  understanding and generation.
\newblock In \emph{{NeurIPS 2019}}, pages 13042--13054, 2019.

\bibitem[Du et~al.(2021)Du, Huang, Dai, Tong, Lepikhin, Xu, Krikun, Zhou, Yu,
  Firat, Zoph, Fedus, Bosma, Zhou, Wang, Wang, Webster, Pellat, Robinson,
  Meier{-}Hellstern, Duke, Dixon, Zhang, Le, Wu, Chen, and Cui]{glam}
Nan Du, Yanping Huang, Andrew~M. Dai, Simon Tong, Dmitry Lepikhin, Yuanzhong
  Xu, Maxim Krikun, Yanqi Zhou, Adams~Wei Yu, Orhan Firat, Barret Zoph, Liam
  Fedus, Maarten Bosma, Zongwei Zhou, Tao Wang, Yu~Emma Wang, Kellie Webster,
  Marie Pellat, Kevin Robinson, Kathy Meier{-}Hellstern, Toju Duke, Lucas
  Dixon, Kun Zhang, Quoc~V. Le, Yonghui Wu, Zhifeng Chen, and Claire Cui.
\newblock Glam: Efficient scaling of language models with mixture-of-experts.
\newblock \emph{CoRR}, abs/2112.06905, 2021.

\bibitem[El{-}Kishky et~al.(2020)El{-}Kishky, Chaudhary, Guzm{\'{a}}n, and
  Koehn]{ccaligned}
Ahmed El{-}Kishky, Vishrav Chaudhary, Francisco Guzm{\'{a}}n, and Philipp
  Koehn.
\newblock {CCAligned}: {A} massive collection of cross-lingual web-document
  pairs.
\newblock In Bonnie Webber, Trevor Cohn, Yulan He, and Yang Liu, editors,
  \emph{Proceedings of the 2020 Conference on Empirical Methods in Natural
  Language Processing, {EMNLP} 2020, Online, November 16-20, 2020}, pages
  5960--5969. Association for Computational Linguistics, 2020.

\bibitem[Fan et~al.(2021)Fan, Bhosale, Schwenk, Ma, El{-}Kishky, Goyal, Baines,
  Celebi, Wenzek, Chaudhary, Goyal, Birch, Liptchinsky, Edunov, Auli, and
  Joulin]{m2m100}
Angela Fan, Shruti Bhosale, Holger Schwenk, Zhiyi Ma, Ahmed El{-}Kishky,
  Siddharth Goyal, Mandeep Baines, Onur Celebi, Guillaume Wenzek, Vishrav
  Chaudhary, Naman Goyal, Tom Birch, Vitaliy Liptchinsky, Sergey Edunov,
  Michael Auli, and Armand Joulin.
\newblock Beyond english-centric multilingual machine translation.
\newblock \emph{J. Mach. Learn. Res.}, 22:\penalty0 107:1--107:48, 2021.

\bibitem[Glorot and Bengio(2010)]{xavier}
Xavier Glorot and Yoshua Bengio.
\newblock Understanding the difficulty of training deep feedforward neural
  networks.
\newblock In Yee~Whye Teh and D.~Mike Titterington, editors, \emph{{AISTATS}
  2010}, volume~9 of \emph{{JMLR} Proceedings}, pages 249--256. JMLR.org, 2010.

\bibitem[Goyal et~al.(2021)Goyal, Gao, Chaudhary, Chen, Wenzek, Ju, Krishnan,
  Ranzato, Guzm{\'{a}}n, and Fan]{flores101}
Naman Goyal, Cynthia Gao, Vishrav Chaudhary, Peng{-}Jen Chen, Guillaume Wenzek,
  Da~Ju, Sanjana Krishnan, Marc'Aurelio Ranzato, Francisco Guzm{\'{a}}n, and
  Angela Fan.
\newblock The {FLORES-101} evaluation benchmark for low-resource and
  multilingual machine translation.
\newblock \emph{CoRR}, abs/2106.03193, 2021.

\bibitem[He et~al.(2016)He, Zhang, Ren, and Sun]{resnet}
Kaiming He, Xiangyu Zhang, Shaoqing Ren, and Jian Sun.
\newblock Deep residual learning for image recognition.
\newblock In \emph{2016 IEEE Conference on Computer Vision and Pattern
  Recognition (CVPR)}, pages 770--778, 2016.
\newblock \doi{10.1109/CVPR.2016.90}.

\bibitem[Huang et~al.(2020)Huang, P{\'{e}}rez, Ba, and
  Volkovs]{XiaoShiHuang2020ImprovingTO}
Xiao~Shi Huang, Felipe P{\'{e}}rez, Jimmy Ba, and Maksims Volkovs.
\newblock Improving transformer optimization through better initialization.
\newblock In \emph{{ICML} 2020}, volume 119 of \emph{Proceedings of Machine
  Learning Research}, pages 4475--4483, 2020.

\bibitem[Huang et~al.(2019)Huang, Cheng, Bapna, Firat, Chen, Chen, Lee, Ngiam,
  Le, Wu, and Chen]{gpipe}
Yanping Huang, Youlong Cheng, Ankur Bapna, Orhan Firat, Dehao Chen, Mia~Xu
  Chen, HyoukJoong Lee, Jiquan Ngiam, Quoc~V. Le, Yonghui Wu, and Zhifeng Chen.
\newblock Gpipe: Efficient training of giant neural networks using pipeline
  parallelism.
\newblock In \emph{NeurIPS 2019}, pages 103--112, 2019.

\bibitem[Jumper et~al.(2021)Jumper, Evans, Pritzel, Green, Figurnov,
  Ronneberger, Tunyasuvunakool, Bates, {\v{Z}}{\'\i}dek, Potapenko, Bridgland,
  Meyer, Kohl, Ballard, Cowie, Romera-Paredes, Nikolov, Jain, Adler, Back,
  Petersen, Reiman, Clancy, Zielinski, Steinegger, Pacholska, Berghammer,
  Bodenstein, Silver, Vinyals, Senior, Kavukcuoglu, Kohli, and
  Hassabis]{AlphaFold2021}
John Jumper, Richard Evans, Alexander Pritzel, Tim Green, Michael Figurnov,
  Olaf Ronneberger, Kathryn Tunyasuvunakool, Russ Bates, Augustin
  {\v{Z}}{\'\i}dek, Anna Potapenko, Alex Bridgland, Clemens Meyer, Simon A~A
  Kohl, Andrew~J Ballard, Andrew Cowie, Bernardino Romera-Paredes, Stanislav
  Nikolov, Rishub Jain, Jonas Adler, Trevor Back, Stig Petersen, David Reiman,
  Ellen Clancy, Michal Zielinski, Martin Steinegger, Michalina Pacholska, Tamas
  Berghammer, Sebastian Bodenstein, David Silver, Oriol Vinyals, Andrew~W
  Senior, Koray Kavukcuoglu, Pushmeet Kohli, and Demis Hassabis.
\newblock Highly accurate protein structure prediction with {AlphaFold}.
\newblock \emph{Nature}, 596\penalty0 (7873):\penalty0 583--589, 2021.
\newblock \doi{10.1038/s41586-021-03819-2}.

\bibitem[Kingma and Ba(2015)]{adam}
Diederik~P. Kingma and Jimmy Ba.
\newblock Adam: {A} method for stochastic optimization.
\newblock In \emph{{ICLR} 2015}, 2015.

\bibitem[Lepikhin et~al.(2021)Lepikhin, Lee, Xu, Chen, Firat, Huang, Krikun,
  Shazeer, and Chen]{gshard}
Dmitry Lepikhin, HyoukJoong Lee, Yuanzhong Xu, Dehao Chen, Orhan Firat, Yanping
  Huang, Maxim Krikun, Noam Shazeer, and Zhifeng Chen.
\newblock Gshard: Scaling giant models with conditional computation and
  automatic sharding.
\newblock In \emph{{ICLR} 2021}, 2021.

\bibitem[Lin et~al.(2021)Lin, Mihaylov, Artetxe, Wang, Chen, Simig, Ott, Goyal,
  Bhosale, Du, Pasunuru, Shleifer, Koura, Chaudhary, O'Horo, Wang, Zettlemoyer,
  Kozareva, Diab, Stoyanov, and Li]{xglm}
Xi~Victoria Lin, Todor Mihaylov, Mikel Artetxe, Tianlu Wang, Shuohui Chen,
  Daniel Simig, Myle Ott, Naman Goyal, Shruti Bhosale, Jingfei Du, Ramakanth
  Pasunuru, Sam Shleifer, Punit~Singh Koura, Vishrav Chaudhary, Brian O'Horo,
  Jeff Wang, Luke Zettlemoyer, Zornitsa Kozareva, Mona~T. Diab, Veselin
  Stoyanov, and Xian Li.
\newblock Few-shot learning with multilingual language models.
\newblock \emph{CoRR}, abs/2112.10668, 2021.

\bibitem[Liu et~al.(2020)Liu, Liu, Gao, Chen, and
  Han]{LiyuanLiu2020UnderstandingTD}
Liyuan Liu, Xiaodong Liu, Jianfeng Gao, Weizhu Chen, and Jiawei Han.
\newblock Understanding the difficulty of training transformers.
\newblock In \emph{Proceedings of the 2020 Conference on Empirical Methods in
  Natural Language Processing (EMNLP)}, pages 5747--5763, 2020.

\bibitem[Ma et~al.(2021)Ma, Dong, Huang, Zhang, Muzio, Singhal, Awadalla, Song,
  and Wei]{deltalm}
Shuming Ma, Li~Dong, Shaohan Huang, Dongdong Zhang, Alexandre Muzio, Saksham
  Singhal, Hany~Hassan Awadalla, Xia Song, and Furu Wei.
\newblock {DeltaLM}: Encoder-decoder pre-training for language generation and
  translation by augmenting pretrained multilingual encoders.
\newblock \emph{CoRR}, abs/2106.13736, 2021.

\bibitem[Nguyen and Salazar(2019)]{ToanQNguyen2019TransformersWT}
Toan~Q. Nguyen and Julian Salazar.
\newblock Transformers without tears: Improving the normalization of
  self-attention.
\newblock \emph{CoRR}, abs/1910.05895, 2019.

\bibitem[Post(2018)]{sacrebleu}
Matt Post.
\newblock A call for clarity in reporting {BLEU} scores.
\newblock In Ondrej Bojar, Rajen Chatterjee, Christian Federmann, Mark Fishel,
  Yvette Graham, Barry Haddow, Matthias Huck, Antonio Jimeno{-}Yepes, Philipp
  Koehn, Christof Monz, Matteo Negri, Aur{\'{e}}lie N{\'{e}}v{\'{e}}ol,
  Mariana~L. Neves, Matt Post, Lucia Specia, Marco Turchi, and Karin Verspoor,
  editors, \emph{Proceedings of the Third Conference on Machine Translation:
  Research Papers, {WMT} 2018, October 31 - November 1, 2018}, pages 186--191.
  Association for Computational Linguistics, 2018.

\bibitem[Qi et~al.(2018)Qi, Sachan, Felix, Padmanabhan, and Neubig]{ted-data}
Ye~Qi, Devendra~Singh Sachan, Matthieu Felix, Sarguna Padmanabhan, and Graham
  Neubig.
\newblock When and why are pre-trained word embeddings useful for neural
  machine translation?
\newblock In Marilyn~A. Walker, Heng Ji, and Amanda Stent, editors,
  \emph{Proceedings of the 2018 Conference of the North American Chapter of the
  Association for Computational Linguistics: Human Language Technologies,
  NAACL-HLT, New Orleans, Louisiana, USA, June 1-6, 2018, Volume 2 (Short
  Papers)}, pages 529--535. Association for Computational Linguistics, 2018.

\bibitem[Radford et~al.(2019)Radford, Wu, Child, Luan, Amodei, and
  Sutskever]{gpt-2}
Alec Radford, Jeff Wu, Rewon Child, David Luan, Dario Amodei, and Ilya
  Sutskever.
\newblock Language models are unsupervised multitask learners.
\newblock 2019.

\bibitem[Rae et~al.(2021)Rae, Borgeaud, Cai, Millican, Hoffmann, Song,
  Aslanides, Henderson, Ring, Young, Rutherford, Hennigan, Menick, Cassirer,
  Powell, van~den Driessche, Hendricks, Rauh, Huang, Glaese, Welbl, Dathathri,
  Huang, Uesato, Mellor, Higgins, Creswell, McAleese, Wu, Elsen, Jayakumar,
  Buchatskaya, Budden, Sutherland, Simonyan, Paganini, Sifre, Martens, Li,
  Kuncoro, Nematzadeh, Gribovskaya, Donato, Lazaridou, Mensch, Lespiau,
  Tsimpoukelli, Grigorev, Fritz, Sottiaux, Pajarskas, Pohlen, Gong, Toyama,
  de~Masson~d'Autume, Li, Terzi, Mikulik, Babuschkin, Clark, de~Las~Casas, Guy,
  Jones, Bradbury, Johnson, Hechtman, Weidinger, Gabriel, Isaac, Lockhart,
  Osindero, Rimell, Dyer, Vinyals, Ayoub, Stanway, Bennett, Hassabis,
  Kavukcuoglu, and Irving]{gopher}
Jack~W. Rae, Sebastian Borgeaud, Trevor Cai, Katie Millican, Jordan Hoffmann,
  H.~Francis Song, John Aslanides, Sarah Henderson, Roman Ring, Susannah Young,
  Eliza Rutherford, Tom Hennigan, Jacob Menick, Albin Cassirer, Richard Powell,
  George van~den Driessche, Lisa~Anne Hendricks, Maribeth Rauh, Po{-}Sen Huang,
  Amelia Glaese, Johannes Welbl, Sumanth Dathathri, Saffron Huang, Jonathan
  Uesato, John Mellor, Irina Higgins, Antonia Creswell, Nat McAleese, Amy Wu,
  Erich Elsen, Siddhant~M. Jayakumar, Elena Buchatskaya, David Budden, Esme
  Sutherland, Karen Simonyan, Michela Paganini, Laurent Sifre, Lena Martens,
  Xiang~Lorraine Li, Adhiguna Kuncoro, Aida Nematzadeh, Elena Gribovskaya,
  Domenic Donato, Angeliki Lazaridou, Arthur Mensch, Jean{-}Baptiste Lespiau,
  Maria Tsimpoukelli, Nikolai Grigorev, Doug Fritz, Thibault Sottiaux, Mantas
  Pajarskas, Toby Pohlen, Zhitao Gong, Daniel Toyama, Cyprien
  de~Masson~d'Autume, Yujia Li, Tayfun Terzi, Vladimir Mikulik, Igor
  Babuschkin, Aidan Clark, Diego de~Las~Casas, Aurelia Guy, Chris Jones, James
  Bradbury, Matthew Johnson, Blake~A. Hechtman, Laura Weidinger, Iason Gabriel,
  William~S. Isaac, Edward Lockhart, Simon Osindero, Laura Rimell, Chris Dyer,
  Oriol Vinyals, Kareem Ayoub, Jeff Stanway, Lorrayne Bennett, Demis Hassabis,
  Koray Kavukcuoglu, and Geoffrey Irving.
\newblock Scaling language models: Methods, analysis {\&} insights from
  training gopher.
\newblock \emph{CoRR}, abs/2112.11446, 2021.

\bibitem[Raffel et~al.(2020)Raffel, Shazeer, Roberts, Lee, Narang, Matena,
  Zhou, Li, and Liu]{t5}
Colin Raffel, Noam Shazeer, Adam Roberts, Katherine Lee, Sharan Narang, Michael
  Matena, Yanqi Zhou, Wei Li, and Peter~J. Liu.
\newblock Exploring the limits of transfer learning with a unified text-to-text
  transformer.
\newblock \emph{Journal of Machine Learning Research}, 21\penalty0
  (140):\penalty0 1--67, 2020.

\bibitem[Schwenk et~al.(2021)Schwenk, Wenzek, Edunov, Grave, Joulin, and
  Fan]{ccmatrix}
Holger Schwenk, Guillaume Wenzek, Sergey Edunov, Edouard Grave, Armand Joulin,
  and Angela Fan.
\newblock {CCMatrix}: Mining billions of high-quality parallel sentences on the
  web.
\newblock In Chengqing Zong, Fei Xia, Wenjie Li, and Roberto Navigli, editors,
  \emph{Proceedings of the 59th Annual Meeting of the Association for
  Computational Linguistics and the 11th International Joint Conference on
  Natural Language Processing, {ACL/IJCNLP} 2021, (Volume 1: Long Papers),
  Virtual Event, August 1-6, 2021}, pages 6490--6500. Association for
  Computational Linguistics, 2021.

\bibitem[Shleifer et~al.(2021)Shleifer, Weston, and Ott]{Normformer2021}
Sam Shleifer, Jason Weston, and Myle Ott.
\newblock Normformer: Improved transformer pretraining with extra
  normalization.
\newblock \emph{CoRR}, abs/2110.09456, 2021.

\bibitem[Smith et~al.(2022)Smith, Patwary, Norick, LeGresley, Rajbhandari,
  Casper, Liu, Prabhumoye, Zerveas, Korthikanti, Zheng, Child, Aminabadi,
  Bernauer, Song, Shoeybi, He, Houston, Tiwary, and Catanzaro]{mt-nlg}
Shaden Smith, Mostofa Patwary, Brandon Norick, Patrick LeGresley, Samyam
  Rajbhandari, Jared Casper, Zhun Liu, Shrimai Prabhumoye, George Zerveas,
  Vijay Korthikanti, Elton Zheng, Rewon Child, Reza~Yazdani Aminabadi, Julie
  Bernauer, Xia Song, Mohammad Shoeybi, Yuxiong He, Michael Houston, Saurabh
  Tiwary, and Bryan Catanzaro.
\newblock Using deepspeed and megatron to train megatron-turing {NLG} 530b, {A}
  large-scale generative language model.
\newblock \emph{CoRR}, abs/2201.11990, 2022.

\bibitem[Vaswani et~al.(2017)Vaswani, Shazeer, Parmar, Uszkoreit, Jones, Gomez,
  Kaiser, and Polosukhin]{transformer}
Ashish Vaswani, Noam Shazeer, Niki Parmar, Jakob Uszkoreit, Llion Jones,
  Aidan~N. Gomez, Lukasz Kaiser, and Illia Polosukhin.
\newblock Attention is all you need.
\newblock In \emph{{NeurIPS} 2017}, pages 5998--6008, 2017.

\bibitem[Wang et~al.(2019)Wang, Li, Xiao, Zhu, Li, Wong, and
  Chao]{Wang2019DLCL}
Qiang Wang, Bei Li, Tong Xiao, Jingbo Zhu, Changliang Li, Derek~F. Wong, and
  Lidia~S. Chao.
\newblock Learning deep transformer models for machine translation.
\newblock In \emph{{ACL} 2019}, pages 1810--1822, 2019.

\bibitem[Wang et~al.(2021)Wang, Bao, Dong, and Wei]{vlmo}
Wenhui Wang, Hangbo Bao, Li~Dong, and Furu Wei.
\newblock Vlmo: Unified vision-language pre-training with
  mixture-of-modality-experts.
\newblock \emph{ArXiv}, abs/2111.02358, 2021.

\bibitem[Xiong et~al.(2020)Xiong, Yang, He, Zheng, Zheng, Xing, Zhang, Lan,
  Wang, and Liu]{RuibinXiong2020OnLN}
Ruibin Xiong, Yunchang Yang, Di~He, Kai Zheng, Shuxin Zheng, Chen Xing,
  Huishuai Zhang, Yanyan Lan, Liwei Wang, and Tie{-}Yan Liu.
\newblock On layer normalization in the transformer architecture.
\newblock In \emph{{ICML} 2020}, volume 119 of \emph{Proceedings of Machine
  Learning Research}, pages 10524--10533, 2020.

\bibitem[Xu et~al.(2021)Xu, Kumar, Yang, Zi, Tang, Huang, Cheung, Prince, and
  Cao]{PengXu2021OptimizingDT}
Peng Xu, Dhruv Kumar, Wei Yang, Wenjie Zi, Keyi Tang, Chenyang Huang, Jackie
  Chi~Kit Cheung, Simon J.~D. Prince, and Yanshuai Cao.
\newblock Optimizing deeper transformers on small datasets.
\newblock In \emph{{ACL/IJCNLP} 2021}, pages 2089--2102, 2021.

\bibitem[Zhang et~al.(2019{\natexlab{a}})Zhang, Titov, and
  Sennrich]{BiaoZhang2019ImprovingDT}
Biao Zhang, Ivan Titov, and Rico Sennrich.
\newblock Improving deep transformer with depth-scaled initialization and
  merged attention.
\newblock In \emph{{EMNLP-IJCNLP} 2019}, pages 898--909, 2019{\natexlab{a}}.

\bibitem[Zhang et~al.(2020)Zhang, Williams, Titov, and Sennrich]{opus100}
Biao Zhang, Philip Williams, Ivan Titov, and Rico Sennrich.
\newblock Improving massively multilingual neural machine translation and
  zero-shot translation.
\newblock In \emph{{ACL} 2020}, pages 1628--1639. Association for Computational
  Linguistics, 2020.

\bibitem[Zhang et~al.(2019{\natexlab{b}})Zhang, Dauphin, and
  Ma]{HongyiZhang2019FixupIR}
Hongyi Zhang, Yann~N. Dauphin, and Tengyu Ma.
\newblock Fixup initialization: Residual learning without normalization.
\newblock In \emph{{ICLR} 2019}, 2019{\natexlab{b}}.

\end{thebibliography}

\newpage

\appendix

\section{Main Theorem Proof}
\subsection{Proof of \cref{lem:softmax}}

\begin{lemma}
     Given $\mathbf{X} = (\mathbf{x_1}, \mathbf{x_2}, ... \mathbf{x_n})^T \in \mathbf{R}^{n \times d}$, where $var(\mathbf{x_i}) = 1$, $mean(\mathbf{x_i}) = 0$ and $q_i \in \mathbf{R}$ for all $i \in [1, n]$, it satisfies that
     \begin{equation*}
         softmax(q_1, q_2, ..., q_n) \mathbf{X} \overset{\Theta}{=} \mathbf{x_i},
     \end{equation*}
     where $\overset{\Theta}{=}$ stands for equal bound of magnitude.
\end{lemma}

\begin{proof}
The weight $s_i$ of $\mathbf{x_i}$ to output is $s_i = \cfrac{e^{q_i}}{e^{\sum_{j=1}^n q_j}}$, $\sum_{i=1}^n s_i = 1$. 

\begin{equation}
    || \text{softmax}(q_1, q_2, ... ,q_n)\mathbf{X} || = 
    || \sum_{i=1}^n s_i \mathbf{x_i} || \le \sum_{i=1}^n s_i || \mathbf{x_i}||
\end{equation}

With $var(\mathbf{x_i}) = 1$, $mean(\mathbf{x_i}) = 0$, for all $i \in [1, n]$, we have $||\mathbf{x_i} || = d$. Therefore, $||\text{softmax}(q_1, q_2, ... ,q_n)\mathbf{X}|| \le ||\mathbf{x_i}|| = d$, which is equivalent to  softmax$(q_1, q_2, ..., q_n) \mathbf{X} \overset{\Theta}{=} \mathbf{x_i}$.
\end{proof}

\subsection{Proof of \cref{thm:encoder}}
\label{proof:thm1}

\begin{theorem}
Given an $N$-layer \our{} $F(x, \mathbf{\theta})$ ($\theta = \{\theta_1, \theta_2, ..., \theta_{2N} \} $), where $\theta_{2l-1}$ and ${\theta_{2l}}$ denote the parameters of self-attention and FFN in $l$-th layer, and each sub-layer is normalized with \deepnorm{}: $x_{l+1} = LN(\alpha x_l + G_l(x_l, \theta_l))$, $ ||\Delta F||$ satisfies: 
\begin{equation*}
    || \Delta F || \le \sum_{i=1}^{2N} \cfrac{\sqrt{v_{i}^2 + w_{i}^2}}{\alpha} || \theta_{i}^* - \theta_{i} ||
\end{equation*}
\end{theorem}

\begin{proof}
Our aim is to study the magnitude of model updates. Following \cite{HongyiZhang2019FixupIR}, we make the following assumptions to simplify the derivations:

\begin{enumerate}
    \item Hidden dimension $d$ equals to $1$.
    \item $var(x + G_l(x)) \overset{\Theta}{=} var(x) + var(G_l(x))$
    \item All relevant weights $v, w$ are positive with magnitude less than 1 and $\alpha, \beta$ for \deepnorm{} are positive with magnitude greater than 1.
\end{enumerate}

Given Assumption 1, if $G(x)$ is feed-forward network with $\theta = \{v, w\}$, then $G(x) \overset{\Theta}{=} v w x$. According to \cref{lem:softmax}, the query and key projections do not change the bound of the attention output's magnitude. Therefore, if $G(x)$ is self-attention with $\theta = \{ q, k, v, w\}$, then $G(x) \overset{\Theta}{=} v w x$. Especially, if Xavier initialization is used for the projection, then the output can preserve the input variance, which is equivalent to $v = w = 1$. 
With Assumption 2, we have:

\begin{equation}
\label{eq:ffn-}
    x_{l+1} = f_l(x_l, \theta_l) = 
    \cfrac{\alpha x + G_l(x)}{\sqrt{Var(\alpha x + G_l(x))}} 
    \overset{\Theta}{=}
    \cfrac{\alpha + v_l w_l}{\sqrt{\alpha^2 + v_l^2 w_l^2}}x
\end{equation}

With \cref{eq:ffn-}, the magnitude of $\frac{\partial f_l}{\partial x}$ and $\frac{\partial f_l}{\partial \theta_l}$ is bounded by: 
\begin{align}
    \cfrac{\partial f_l}{\partial x} &\overset{\Theta}{=} \cfrac{\alpha + v_l w_l}{\sqrt{\alpha^2 + v_l^2 w_l^2}} \notag \\
    \cfrac{\partial f_l}{\partial \theta_l} &\overset{\Theta}{=} (\cfrac{\partial f}{\partial v_l}, \cfrac{\partial f}{\partial w_l}) \overset{\Theta}{=} \cfrac{\alpha x_l(\alpha - v_l w_l)}{(\alpha^2 + v_l^2 w_l^2)^\frac{3}{2}}\,(w_l, v_l) 
\end{align}

Besides, the model update $||\Delta F||$ satisfies:

\begin{equation}
\label{eq:def}
    ||\Delta F|| = ||F(x, \theta^*) - F(x, \theta)|| = || x_{{2N}+1}^* - x_{{2N}+1} ||  = || f(x_{2N}^*, \theta_{2N}^*) - f(x_{2N}, \theta_{2N}) ||  \\
\end{equation}

Using Taylor expansion for \cref{eq:def}, we get:
\begin{align}
    ||\Delta F|| &= || x_{{2N}+1}^* - x_{{2N}+1} || \\
    &\approx ||\cfrac{\partial f}{\partial x}(x_{2N}, \theta_{2N}) (x_{2N}^* - x_{2N}) + \cfrac{\partial f}{\partial \theta}(x_{2N}, \theta_{2N}) (\theta_{2N}^* - \theta_{2N})^T||  \notag \\
    &\le ||\cfrac{\partial f}{\partial x}(x_{2N}, \theta_{2N})|| \cdot || x_{2N}^* - x_{2N} || + ||\cfrac{\partial f}{\partial \theta}(x_{2N}, \theta_{2N})|| \cdot ||\theta_{2N}^* - \theta_{2N} || \notag \\
    &=\cfrac{\alpha + v_{2N} w_{2N}}{\sqrt{\alpha^2 + v_{2N}^2 w_{2N}^2}}|| x_{2N}^* - x_{2N} || + \cfrac{\alpha(\alpha - v_{2N} w_{2N})}{(\alpha^2 + v_{2N}^2 w_{2N}^2)^\frac{3}{2}} \sqrt{v_{2N}^2 + w_{2N}^2} ||\theta_{2N}^* - \theta_{2N} || \notag \\
    &\approx ||x_{2N}^* - x_{2N} || + \cfrac{\sqrt{v_{2N}^2 + w_{2N}^2}}{\alpha} || \theta_{2N}^* - \theta_{2N} ||
\end{align}

Then, we have:
\begin{equation}
    || x_{2N+1}^* - x_{2N+1} || \le \sum_{i=1}^{2N} \cfrac{\sqrt{v_{i}^2 + w_{i}^2}}{\alpha} || \theta_{i}^* - \theta_{i} ||
\end{equation}
\end{proof}

For vanilla \postln{} with standard initialization, $\alpha = v_i = w_i = 1$, so $||\Delta F|| = \mathcal{O}(\sum_{i=1}^{2N} ||\theta_i^* - \theta_i ||)$.

\textbf{Proof of \cref{thm: en2de}}

\begin{theorem}
Given an encoder-decoder \our{} $F_{ed}(x, y, \theta_e, \theta_d)$ with N encoder layers and M decoder layers, where each encoder sub-layer is normalized as $x_{l+1} = LN(\alpha_e x_l + G_{el}(x_l, \theta_{el}))$, and the decoder sub-layer is normalized as $x_{l+1} = LN(\alpha_d x_l + G_{dl}(x_l, \theta_{dl}))$, $ ||\Delta F_{ed}||$ satisfies:
\begin{align}
\label{appendix eq: en2de}
    || \Delta F_{ed} || 
    &\le \sum_{j=1}^{M} \cfrac{v_{d,3j-1}w_{d, 3j-1}}{\alpha_d}\sum_{i=1}^{2N} \cfrac{\sqrt{v_{ei}^2 +  w_{ei}^2}}{\alpha_e} || \theta_{ei}^* - \theta_{ei}|| \notag \\
    &\quad + \sum_{j=1}^{3M} \cfrac{\sqrt{v_{dj}^2 + w_{dj}^2}}{\alpha_d} ||\theta_{dj}^* - \theta_{dj}||
\end{align}
\end{theorem}

\begin{proof}
The derivation of self-attention and FFN layers is given in \cref{proof:thm1}. For the cross-attention layers, we have:

\begin{equation}
\label{eq:cross-}
    y_{l+1} = f_{dl}(y_{l}, x_e, \theta_{dl}) = \cfrac{\alpha_d y_{l} + G_l(x_e, y_l)}{\sqrt{Var(\alpha_d y_{l} + G_{dl}(x_e, y_l))}} \overset{\Theta}{=} \cfrac{\alpha_d y_l + v_l w_l x_e}{\sqrt{\alpha_d^2 + v_l^2w_l^2}}
\end{equation}
With \cref{eq:cross-}, we have the bound of the derivative of $f_{dl}$:

\begin{align*}
    \cfrac{\partial f_{dl}}{\partial y} &\overset{\Theta}{=} \cfrac{\alpha_d}{\sqrt{\alpha_d^2 + v_l^2 w_l^2}},\quad 
    \cfrac{\partial f_{dl}}{\partial x_e} \overset{\Theta}{=} \cfrac{v_l w_l}{\sqrt{\alpha_d^2 + v_l^2 w_l^2}} \\ \\
    \cfrac{\partial f_{dl}}{\partial \theta_{dl}} &\overset{\Theta}{=} (\cfrac{\partial f_{dl}}{\partial v_{dl}}, \cfrac{\partial f_{dl}}{\partial w_{dl}}) \overset{\Theta}{=} \cfrac{\alpha_d x_e(\alpha_d - v_{dl} w_{dl})}{(\alpha_d^2 + v_{dl}^2 w_{dl}^2)^\frac{3}{2}}\,(w_{dl}, v_{dl})
\end{align*}

By means of Taylor expansion, we estimate the update of $l$-th cross-attention layer $||y_{l+1}^* - y_{l+1} ||$ as:

\begin{align}
    ||y_{l+1}^* - y_{l+1} || 
    &= ||f_{dl}^*(y_{l}^*, x_{2N+1}^*, \theta_{dl}^*) - f_{dl}(y_{l}, x_{2N+1}, \theta_{dl}) || \notag \\
    &\approx \cfrac{\alpha_d}{\sqrt{\alpha_d^2 + v_{dl}^2 w_{dl}^2}} || y_{l}^* - y_{l} || + \cfrac{v_{dl} w_{dl}}{\sqrt{\alpha_d^2 + v_{dl}^2 w_{dl}^2}} ||x_{2N+1}^* - x_{2N+1}|| \notag \\
    &\quad+ \cfrac{\alpha_d(\alpha_d - v_{dl} w_{dl})}{(\alpha_d^2 + v_{dl}^2 w_{dl}^2)^\frac{3}{2}}\,\sqrt{v_{dl}^2 + w_{dl}^2}||\theta_{dl}^* - \theta_{dl}|| \notag \\
    &\le|| y_{l}^* - y_{l} || + \cfrac{v_{dl} w_{dl}}{\alpha_d}||x_{2N+1}^* - x_{2N+1}|| + \cfrac{\sqrt{v_{dl}^2 + w_{dl}^2}}{\alpha_d} ||\theta_{dl}^* - \theta_{dl}||
\end{align}

According to \cref{thm:encoder}, we have $||x_{2N+1}^* - x_{2N+1}|| = \mathcal{O}(\sum_{i=1}^{2N} \frac{\sqrt{v_{ei}^2 + w_{ei}^2}}{\alpha_e} || \theta_{ei}^* - \theta_{ei}||)$. Therefore, the magnitude of $|| \Delta F_{ed} ||$ satisfies:

\begin{equation}
    || \Delta F_{ed} || \le \sum_{j=1}^{M} \cfrac{v_{d, 3j-1} w_{d, 3j-1}}{\alpha_d}\sum_{i=1}^{2N} \cfrac{\sqrt{v_{ei}^2 + w_{ei}^2}}{\alpha_e} || \theta_{ei}^* - \theta_{ei}|| + \sum_{j=1}^{3M} \cfrac{\sqrt{v_{dj}^2 + w_{dj}^2}}{\alpha_d} ||\theta_{dj}^* - \theta_{dj}||
    \label{ap:thm2}
\end{equation}
\end{proof}

As a special case, the corresponding parameters in~\cref{ap:thm2} for vanilla \postln{} with standard initialization are $1$, so its model update $||\Delta F_{ed} || = \mathcal{O}(M\sum_{i=1}^{2N} || \theta_{ei}^* - \theta_{ei} || + \sum_{j=1}^{3M}|| \theta_{dj}^* - \theta_{dj} ||)$.

\section{Derivation for Encoder-Decoder Architecture}
\label{appendix_derivation_en2de}
Here, we give the derivation of \our{} for the encoder-decoder architecture with an $N$-layer encoder and an $M$-layer decoder.
As in \cref{set:imple}, we have $v_d = w_d = (12M)^{-\frac{1}{4}}$, $\alpha_d = (3M)^{\frac{1}{4}}$ to bound the second term of \cref{ap:thm2} to $\Theta(\eta)$. For the first term, we set $v_{ei} = v_e, w_{ei} = w_e$, so that it goes to:
\begin{align}
    \sum_{j=1}^{M} \cfrac{v_{d, 3j-1} w_{d, 3j-1}}{\alpha_d}\sum_{i=1}^{2N} \cfrac{\sqrt{v_{ei}^2 + w_{ei}^2}}{\alpha_e} || \theta_{ei}^* - \theta_{ei}|| 
    &= M \cfrac{(12M)^{-\frac{1}{2}}}{(3M)^\frac{1}{4}} \sum_{i=1}^{2N} \cfrac{\sqrt{v_{ei}^2 + w_{ei}^2}}{\alpha_e} || \theta_{ei}^* - \theta_{ei}|| \\
    &\overset{\Theta}{=}\eta\left( \cfrac{N^4 M}{27} \right)^{\frac{1}{4}} \cfrac{v_e^2 + w_e^2}{\alpha_e^2}
\end{align}

In this work, we use $\alpha_e^2 = (N^4 M / 27)^{\frac{1}{8}}$, $v_e^2 + w_e^2 = (N^4 M / 27)^{- \frac{1}{8}}$ and $v_e = w_e = \beta_e$ that is $\alpha_e = 0.81(N^4 M)^{\frac{1}{16}}$, $\beta_e = 0.87(N^4 M)^{-\frac{1}{16}}$ to satisfy the condition.

\section{Derivation for Encoder-only (Decoder-only) Architecture}
\label{appendix_derivation}
For an $N$-layer \our{}, starting from \cref{thm:encoder} we have,
\begin{align}
    || x_{2N+1}^* - x_{2N+1} || 
    \le \sum_{i=1}^{2N} \cfrac{\sqrt{v_{i}^2 + w_{i}^2}}{\alpha} || \theta_{i}^* - \theta_{i} || \notag \\
    \le \eta \sum_{i=1}^{2N} \cfrac{\sqrt{v_{i}^2 + w_{i}^2}}{\alpha} || \cfrac{\partial \mathcal{L}}{\partial F}||\cdot||\cfrac{\partial F}{\partial \theta_i} ||
\end{align}
By assumption $||\frac{\partial \mathcal{L}}{\partial F} || = \mathcal{O}(1)$, and $||\frac{\partial F}{\partial \theta_i} || \le ||\frac{\partial F}{\partial \theta_{2N}} || \overset{\Theta}{=} \frac{||\theta_{2N}||}{\alpha}$, we achieve:
\begin{align}
    \sum_{i=1}^{2N} \cfrac{\sqrt{v_{i}^2 + w_{i}^2}}{\alpha} || \cfrac{\partial \mathcal{L}}{\partial F}||\cdot||\cfrac{\partial F}{\partial \theta_i} || 
    \le  \mathcal{O}(\cfrac{\sqrt{v_{2N}^2 + w_{2N}^2}}{\alpha}\sum_{i=1}^{2N}\cfrac{\sqrt{v_i^2 + w_i^2}}{\alpha}) = \mathcal{O}(1)
\end{align}
Due to symmetry, we set $v_i = v, w_j = w$, so it goes to $2N\frac{v^2 + w^2}{\alpha^2} = 1$. In this work, we use $v = w = (8N)^{-\frac{1}{4}}$ and $\alpha = (2N)^{\frac{1}{4}}$ to satisfy the condition.

\section{Experimental Details}

\subsection{Hyperparameters for IWSLT-14 De-En}

\begin{table*}[htbp]
\begin{center}
\begin{tabular}{l|c}
\toprule
\textbf{Hyperparameters} &  \textbf{Value} \\
\midrule
Learning rate & 5e-4 \\
Learning rate scheduler & inverse sqrt \\
Warm-up updates & 4000 \\
Warm-up init learning rate & 1e-7 \\
Max tokens & 4000 \\
Adam $\epsilon$ & 1e-8 \\
Adam $\beta$ & (0.9, 0.98) \\
Label smoothing & 0.1 \\
Training updates & 8K \\
\midrule
Gradient clipping & 0.0 \\
Dropout & 0.4 \\
Weight decay & 0.0001 \\
\midrule
Hidden size & 512 \\
FFN inner hidden size & 2048 \\
Attention heads & 8 \\
\bottomrule
\end{tabular}
\caption{Hyperparameters for the machine translation experiments on the IWSLT-14 De-En dataset.}
\end{center}
\end{table*}

\subsection{Hyperparameters for WMT-17 En-De}

\begin{table*}[htbp]
\begin{center}
\begin{tabular}{l|cccc}
\toprule
\textbf{Hyperparameters} &  \textbf{No-LN} & \textbf{\preln{}} &  \textbf{\postln{}} & \textbf{\deepnorm{}} \\
\midrule
Learning rate & 5e-4 & 1.5e-3 & 1.5e-3 & 1.5e-3 \\
Learning rate scheduler & \multicolumn{4}{c}{inverse sqrt} \\
Warm-up updates & \multicolumn{4}{c}{4000} \\
Warm-up init learning rate & \multicolumn{4}{c}{1e-7} \\
Max tokens & \multicolumn{4}{c}{128 $\times$ 4096} \\
Adam $\epsilon$ & \multicolumn{4}{c}{1e-8} \\
Adam $\beta$ & \multicolumn{4}{c}{(0.9, 0.98)} \\
Label smoothing & \multicolumn{4}{c}{0.1} \\
Training updates & \multicolumn{4}{c}{100K} \\
\midrule
Gradient clipping & \multicolumn{4}{c}{0.0} \\
Dropout & \multicolumn{4}{c}{0.4} \\
Weight decay & \multicolumn{4}{c}{0.0001} \\
\midrule
Hidden size & \multicolumn{4}{c}{512} \\
FFN inner hidden size & \multicolumn{4}{c}{2048} \\
Attention heads & \multicolumn{4}{c}{8} \\
\bottomrule
\end{tabular}
\caption{Hyperparameters for the base-setting experiments on the WMT-17 En-De dataset.}
\end{center}
\end{table*}

\begin{table*}[htbp]
\begin{center}
\begin{tabular}{l|ccc}
\toprule
\textbf{Hyperparameters} &  \textbf{Base size} & \textbf{Medium size} &  \textbf{Large size}\\
\midrule
Hidden size & 512 & 768 & 1,024 \\
FFN inner hidden size & 2048 & 3072 & 4096 \\
Attention heads & 8 & 12 & 16 \\
Layers & \multicolumn{3}{c}{18-18} \\
\midrule
Learning rate & \multicolumn{3}{c}{5e-4} \\
Learning rate scheduler & \multicolumn{3}{c}{inverse sqrt} \\
Warm-up updates & \multicolumn{3}{c}{4000} \\
Warm-up init learning rate & \multicolumn{3}{c}{1e-7} \\
Max tokens & \multicolumn{3}{c}{128 $\times$ 4096} \\
Adam $\epsilon$ & \multicolumn{3}{c}{1e-6} \\
Adam $\beta$ & \multicolumn{3}{c}{(0.9, 0.98)} \\
Label smoothing & \multicolumn{3}{c}{0.1} \\
Training updates & \multicolumn{3}{c}{30K} \\
\midrule
Gradient clipping & \multicolumn{3}{c}{1.0} \\
Dropout & \multicolumn{3}{c}{0.4} \\
Weight decay & \multicolumn{3}{c}{0.0} \\
\bottomrule
\end{tabular}
\caption{Hyperparameters for the large-setting experiments on the WMT-17 En-De dataset.}
\end{center}
\end{table*}

\newpage
\subsection{Hyperparameters for OPUS-100}

\begin{table*}[htbp]
\begin{center}
\begin{tabular}{l|c}
\toprule
\textbf{Hyperparameters} &  \textbf{Value} \\
\midrule
Learning rate & 5e-4 \\
Learning rate scheduler & inverse sqrt \\
Warm-up updates & 4000 \\
Warm-up init learning rate & 1e-7 \\
Max tokens & 128 $\times$ 4096 \\
Adam $\epsilon$ & 1e-8 \\
Adam $\beta$ & (0.9, 0.98) \\
Label smoothing & 0.1 \\
Training epochs & 4 \\
\midrule
Gradient clipping & 0.0 \\
Dropout & 0.1 \\
Weight decay & 0.0 \\
\midrule
Hidden size & 512 \\
FFN inner hidden size & 2048 \\
Attention heads & 8 \\
\bottomrule
\end{tabular}
\caption{Hyperparameters for the machine translation experiments on the OPUS-100 dataset.}
\end{center}
\end{table*}

\newpage
\subsection{Hyperparameters for 102-Language Machine Translation}

\begin{table*}[htbp]
\begin{center}
\begin{tabular}{l|c}
\toprule
\textbf{Hyperparameters} &  \textbf{Value} \\
\midrule
Learning rate & 5e-4 \\
Learning rate scheduler & inverse sqrt \\
Warm-up updates & 6000 \\
Warm-up init learning rate & 1e-7 \\
Max tokens & 256 $\times$ 4096 \\
Adam $\epsilon$ & 1e-6 \\
Adam $\beta$ & (0.9, 0.98) \\
Label smoothing & 0.1 \\
Training updates & 260K \\
\midrule
Gradient clipping & 1.0 \\
Dropout & 0.1 \\
Weight decay & 0.0 \\
\midrule
Hidden size & 1024 \\
FFN inner hidden size & 4096 \\
Attention heads & 16 \\
Layers & 100-100 \\
\bottomrule
\end{tabular}
\caption{Hyperparameters for the machine translation experiments on the 102-language dataset.}
\end{center}
\end{table*}

\subsection{Evaluation Details}

For IWSLT-14 and WMT-17, we use the in-built BLEU scripts of Fairseq to report the scores.
Besides, we report the case-sensitive detokenized BLEU using sacreBLEU~\citep{sacrebleu} for the results of OPUS-100.\footnote{BLEU+case.mixed+lang.\{src\}-\{tgt\}+numrefs.1+smooth.exp+tok.13a+version.1.4.14}

For WMT, OPUS, and TED, we use the same test sets and evaluation scripts as in M2M~\citep{m2m100}, and the results of M2M are directly from the paper~\citep{m2m100}. For the Flores-101 evaluation set, we report the spBLEU\footnote{\url{https://github.com/facebookresearch/flores}} of M2M-12B with the public checkpoint and script.\footnote{\url{https://github.com/pytorch/fairseq/tree/main/examples/m2m_100}}

\newpage
\section{Experimental Results in \cref{mnmt}}

\begin{figure}[htbp]
  \centering
  \includegraphics[width=1.0\columnwidth]{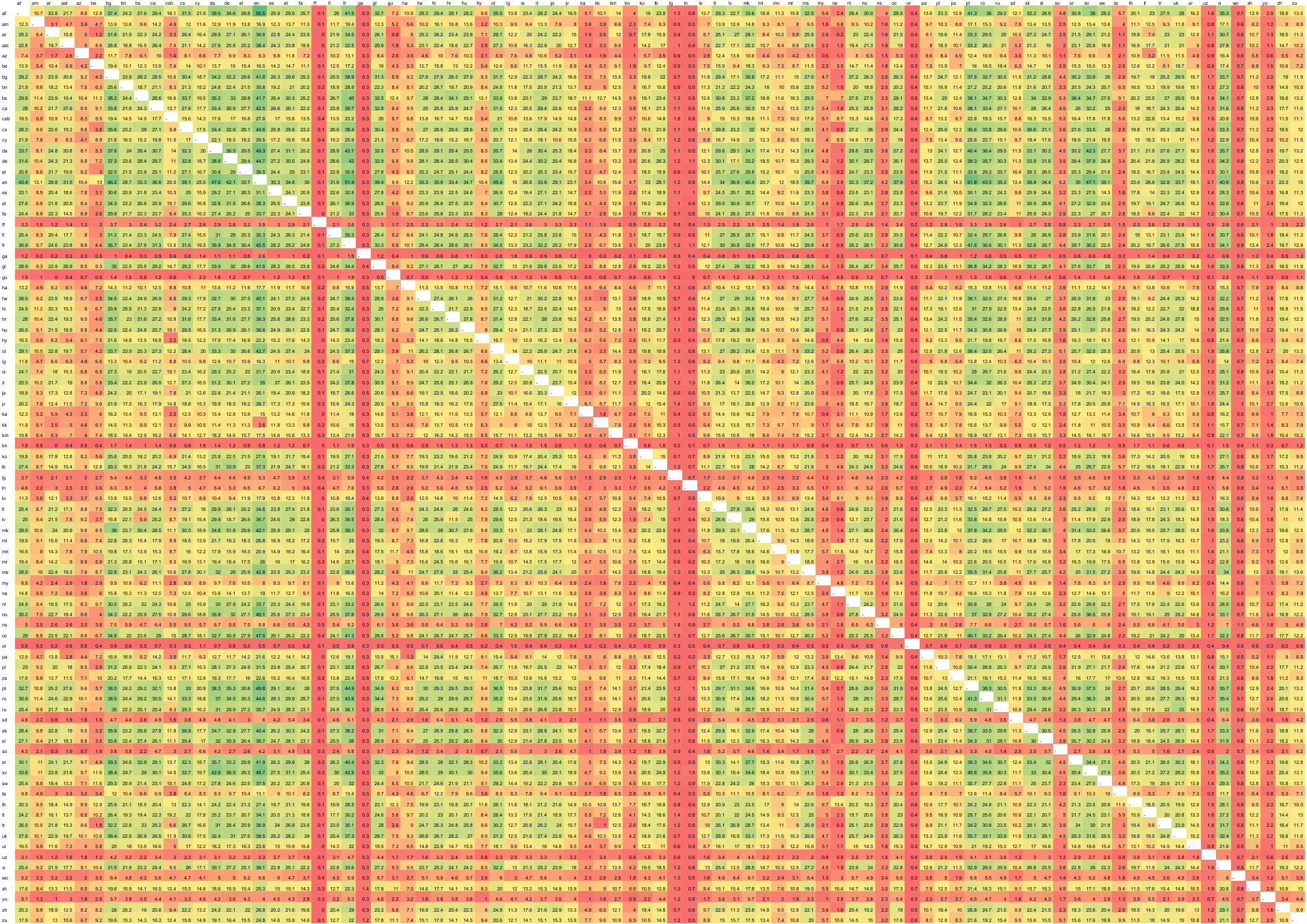}
  \caption{Evaluation results of 12B M2M-100 on a subset of FLORES-101 devtest set. The $i$-th row is the source language, while $j$-th column is the target language. There are 87 languages and 7,482 directions.}
\end{figure}

\begin{figure}[htbp]
  \centering
  \includegraphics[width=1.0\columnwidth]{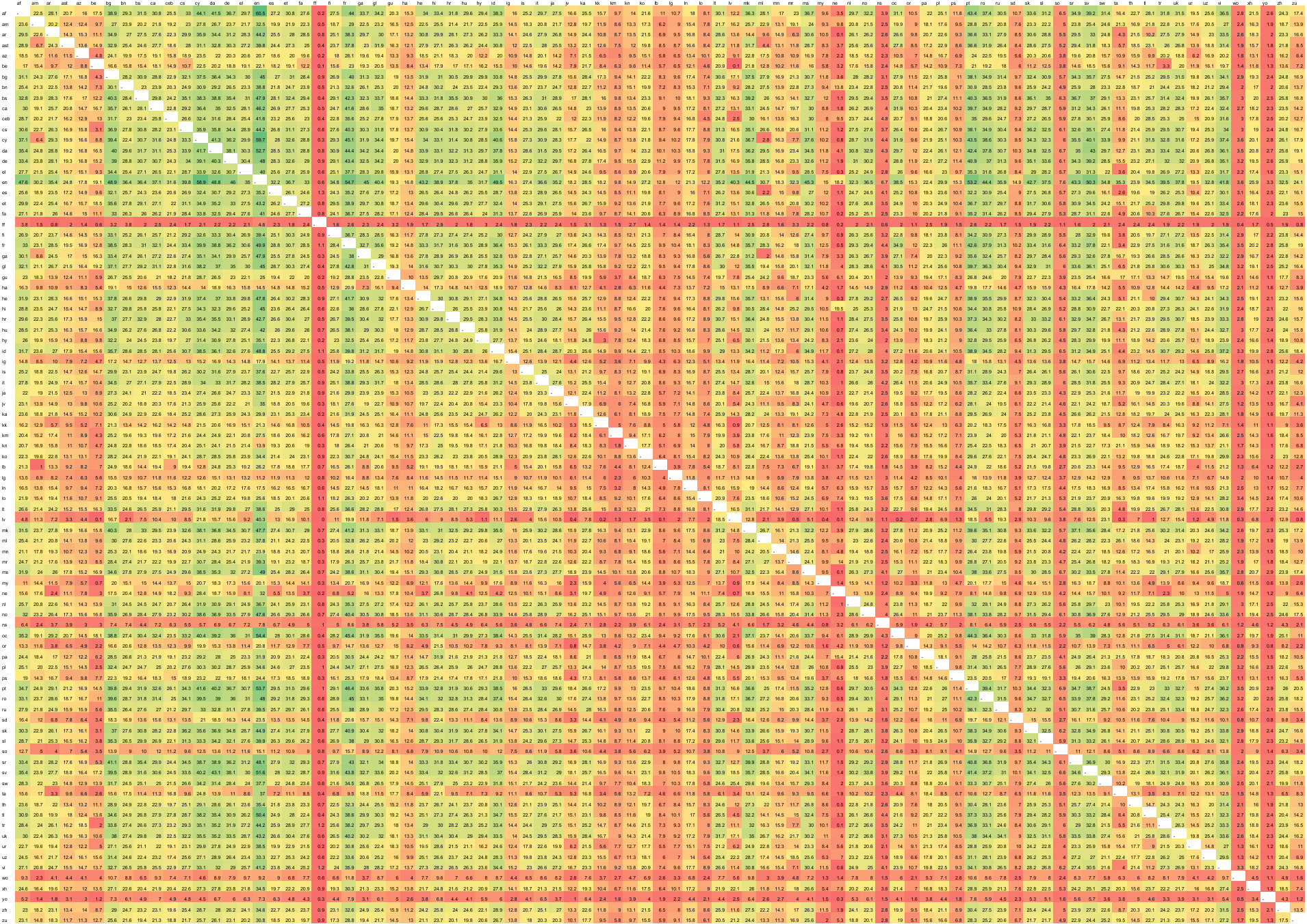}
  \caption{Evaluation results of 3.2B \our{} on a subset of FLORES-101 devtest set. The $i$-th row is the source language, while $j$-th column is the target language. There are 87 languages and 7,482 directions.}
\end{figure}

\end{document}